%% file: 00_neurips_2026.tex
\documentclass{article}

\usepackage[preprint]{neurips_2026}

\usepackage[utf8]{inputenc} 
\usepackage[T1]{fontenc}    
\usepackage{hyperref}       
\usepackage{url}            
\usepackage{booktabs}       
\usepackage{amsfonts}       
\usepackage{nicefrac}       
\usepackage{microtype}      
\usepackage{xcolor}         

\usepackage{enumitem}

\usepackage{amsmath}

\include{__macros}

\title{Validating Causal Abstraction Metrics on \\ Simulated Complex Systems}

\author{
Maxime Méloux$^1$ \quad Tiago Pimentel$^2$ \quad François Portet$^1$ \quad Maxime Peyrard$^1$ \\
$^1$Université Grenoble Alpes, CNRS, Grenoble INP, LIG \quad $^2$ETH Zürich\\
\texttt{\{melouxm, portetf, peyrardm\}@univ-grenoble-alpes.fr}\\
\texttt{tiago.pimentel@inf.ethz.ch}\\
}

\begin{document}

\usetikzlibrary{calc}

\maketitle

\begin{abstract}
A central goal of science is to produce valid explanations of complex systems: high-level causal accounts that faithfully reflect the behavior of lower-level mechanisms. Yet no consensus exists on how to measure whether a proposed high-level explanation is actually valid. We introduce a benchmark of ten complex systems spanning both discrete and continuous state spaces, as well as static and dynamical regimes, each equipped with consensual ground-truth causal explanations and invalid contrastive conditions. Within a unified causal abstraction framework, we systematically evaluate over thirty candidate metrics drawn from observational, functional, information-theoretic, and causal families. Our results show that only the latter reliably discriminates valid from invalid abstractions, and only when incorporating faithfulness testing over unmapped variables. Building on these findings, we introduce the Causal Abstraction Error (CAE), a continuous validity metric with an explicit faithfulness test, which passes all discrimination tests across every system and can converge with as few as 30 sampled interventions. We offer it as a general-purpose metric for the discovery and validation of high-level explanations.\looseness=-1
\end{abstract}

\input{sections/01_introduction}
\input{sections/02_background}
\input{sections/03_cae}
\input{sections/04_benchmark}
\input{sections/05_results}
\input{sections/06_discussion}
\input{sections/07_acknowledgements}

\clearpage
\bibliographystyle{plainnat}
\bibliography{causal_abs}


\clearpage
\appendix

\input{appendix/full_results}

\input{appendix/extended_related_work}
\input{appendix/controlled_experiments}
\input{appendix/theoretical}
\input{appendix/systems}
\input{appendix/baseline_details}

\input{appendix/case_study}



\end{document}

%% file: __macros.tex
\usepackage{tikz}
\usepackage{tikz-cd}
\usetikzlibrary{arrows.meta, positioning}

\usepackage{tabularx}
\usepackage{array}
\usepackage{makecell}
\usepackage{colortbl} 
\usepackage[dvipsnames]{xcolor}
\usepackage{wrapfig}

\definecolor{dynrow}{RGB}{245,248,255}

\newcolumntype{Y}{>{\raggedright\arraybackslash}X}
\newcolumntype{C}[1]{>{\centering\arraybackslash}p{#1}}

\hypersetup{
  colorlinks=true,
  linkcolor=RubineRed,
  citecolor=MidnightBlue,
  urlcolor=MidnightBlue,
  pdfborder={0 0 0}
}

\newtheorem{proposition}{Proposition}[section]
\newtheorem{definition}{Definition}[section]

\newtheorem{proof}{Proof}[section]
\newtheorem{assumption}{Assumption}[section]
\newtheorem{remark}{Remark}[section]

\usepackage[textsize=tiny,disable]{todonotes}
\newcommand{\note}[4][]{\todo[author=#2,color=#3,size=\scriptsize,fancyline,caption={},#1]{#4}} 
\newcommand{\tiago}[2][]{\note[#1]{\textbf{Tiago}}{cyan!30}{#2}}

\newcommand{\cae}{\textnormal{\scriptsize CAE}}
\newcommand{\caeup}{\cae\ensuremath{_{\uparrow}}}
\newcommand{\caedown}{\cae\ensuremath{_{\downarrow}}}
\newcommand{\caenf}{\cae\ensuremath{_{\text{NF}}}}
\newcommand{\caeupnf}{\cae\ensuremath{_{\uparrow\text{NF}}}}
\newcommand{\caedownnf}{\cae\ensuremath{_{\downarrow\text{NF}}}}

%% file: sections/01_introduction.tex
\section{Introduction}

A central goal of science is to produce \textbf{explanations}: not merely descriptions or predictions, but accounts of \textit{why} phenomena occur \citep{hempel1965aspects, woodward2003making}. While the philosophy of science has long debated what explanations are, a working consensus has emerged around a few core desiderata: good scientific explanations should be causally informative \citep{woodward2003making,pearl_causality,salmon1984scientific}, parsimonious \citep{Kitcher1962-KITEUA,batterman2014minimal}, and appropriately scoped to the context and level of description at which the target phenomenon is best characterized \citep{lombrozo2006structure,Potochnik2017-POTIAT-3}.
Finding good explanations is particularly difficult when the system under study exhibits what Warren Weaver famously called \textit{organized complexity} \citep{weaver1991science}: systems with many interacting parts that are neither so disordered as to yield to statistical averaging, nor so simple as to admit direct analytic treatment. 
The appropriate high-level variables for explanation must be discovered, and the class of functions that legitimately aggregate low-level quantities into those variables is itself a subject of inquiry \citep{doi:10.1073/pnas.1314922110,Potochnik2017-POTIAT-3}.

This raises a fundamental challenge: what are useful high-level variables, and how should they be defined from lower-level quantities? Different fields studying different classes of complex systems have developed their own candidate answers, such as firing rates and population codes in systems neuroscience \citep{cunningham2014dimensionality}, or species abundance and trophic levels in ecology \citep{loreau2010populations}.
However, a unified cross-disciplinary methodology for evaluating high-level explanations of complex systems remains elusive. A sobering illustration of why comes from \cite{jonas2017could}, who applied the standard causal and statistical toolkit of neuroscience to a fully observable microprocessor: a relatively simple system engineered for modular, hierarchical organization. In doing so, they failed to recover meaningful high-level properties of this system.
The recent success of artificial intelligence (AI) offers a similar account. Similarly to microprocessors, modern AI systems are structurally complex, yet they are fully observable and perfectly manipulable, a rare property in the natural sciences. 
This privileged epistemic access makes AI systems a methodological laboratory for the science of complex systems \citep{Holtzman2025}, to develop and test explanation discovery procedures in simplified settings. Yet, AI interpretability has also proven sobering, with a large tendency to produce false positive findings and explanations that are unstable and not generalizable \citep{hewitt-liang-2019-designing, ravichander-etal-2021-probing, meloux2025dead}.

A general metric for evaluating proposed high-level explanations of low-level complex systems is therefore a prerequisite to move forward: without it, the discovery and evaluation of candidates cannot be done rigorously, and progress is hard to measure or transfer across domains. Yet, no such agreed-upon, general metric exists. 
A growing body of work has begun addressing this gap through the lens of \textbf{causal abstraction} \citep{Beckers_Halpern_2019}, the idea that a high-level causal model is a valid explanation of a low-level system if interventions at the high level correspond faithfully to interventions at the low level under a postulated abstraction map. This framework is conceptually appealing and promising, but existing implementations differ in important ways: in whether they treat abstraction as exact or approximate \citep{pmlr-v115-beckers20a}, and in how they handle distributions over inputs or over interventions \citep{geiger_causal_2021, JMLR:v26:23-0058}. Current concrete metric implementations developed in the field of AI interpretability have also been shown to suffer from the global issues of striking false positives and non-generalizable findings \citep{sutter2025nonlinearrepresentationdilemmacausal, meloux2025dead}. 

\textbf{In this work,}  we set out to systematically test, compare, and benchmark a wide range of proposed explanation metrics, observational, information-theoretic, symbolic, and causal, across a diverse collection of idealized complex systems with known, consensual high-level explanations. Rather than confining ourselves to a single domain, we deliberately span systems that vary across multiple dimensions: (i) with discrete versus continuous state spaces, (ii) with static versus temporal dynamics, (iii) with disordered versus structured complexity (in Weaver's sense). For each system, we manually construct the alignment function mapping low-level variables to high-level ones, implement all pairs within a shared conceptual formalism (causal abstraction) and a unified programmatic API supporting both observational and interventional queries. Finally, we handcraft contrastive perturbations, modifications to either the low-level or high-level model that render the proposed explanation invalid, allowing us to test whether each candidate metric correctly recognizes broken explanations as well as valid ones. This is an empirical test of the practical operationalization of existing conceptual ideas. 

Our results provide strong support for the causal abstraction research program, but only for specific distributional variants. 
We further introduce the Causal Abstraction Error (\cae{}), an error measure that incorporates a principled treatment of distributional inputs and interventions, a built-in faithfulness measurement, and produces a real-valued degree-of-violation score. The \cae{} passes all contrastive tests across every system class we examine and outperforms related causal abstraction variants.

%% file: sections/02_background.tex
\section{Metrics of Explanation Validity}
\label{sec:background}

Following the framework of \citet{pearl_causality} and \citet{Beckers_Halpern_2019}, we model the studied system as a \emph{structural causal model} (SCM). \citet{NEURIPS2024_26b8e3dc} previously adopted Pearl's modeling assumption that simulated complex systems belong to this class.\looseness=-1

\begin{definition}[Structural Causal Model]
\label{def:scm}
A structural causal model is a tuple $\mathcal{M} = \langle \mathcal{U},\, \mathcal{V},\, \mathcal{R},\, \mathcal{F},\, P_{\mathcal{U}} \rangle$, where $\mathcal{U}$ is a set of \emph{exogenous} variables with joint distribution $P_{\mathcal{U}}$; $\mathcal{V}$ is a set of \emph{endogenous} variables; $\mathcal{R} = \{\mathcal{R}_W\}_{W \in \mathcal{U} \cup \mathcal{V}}$ assigns a range to every variable; and $\mathcal{F} = \{f_V\}_{V \in \mathcal{V}}$ is a set of \emph{structural equations}, each determining $V$ from its direct causes $\mathrm{Pa}(V) \subseteq \mathcal{V}$ and a noise term $\mathcal{U}_V \subseteq \mathcal{U}$. For any subset $\mathcal{W} \subseteq \mathcal{U} \cup \mathcal{V}$, we write $\mathcal{R}_\mathcal{W}$ for its joint range.
\end{definition}

We write $\mathcal{M}(\mathcal{W} \mid u, \mathrm{do})$ for the joint value assumed by the variables $\mathcal{W} \subseteq \mathcal{V}$ in $\mathcal{M}$ under exogenous realization $u \in \mathcal{R}_\mathcal{U}$ and intervention $\mathrm{do}$ (empty $\mathrm{do}$ is the observational case); for a single variable $V$, we abbreviate $\mathcal{M}(V \mid u, \mathrm{do})$. For deterministic models, we omit the vacuous dependence on $u$.

A \emph{computational explanation} is a surrogate model $\mathcal{E}$ taken from some admissible class $\mathfrak{E}$, intended to account for the behavior of $\mathcal{M}$ at a relevant level of abstraction. 
The central question studied in this paper is: \textit{Under what conditions is $\mathcal{E}$ a valid computational explanation of $\mathcal{M}$ under $P_{\mathcal{U}}$?} Radically different answers to this question have been proposed. We briefly discuss here the ones we evaluate and offer a more in-depth description in Appendix~\ref{app:extended_background}. 

\paragraph{Observational validity.}
The most elementary criterion is \emph{observational equivalence}: $\mathcal{E}$ should reproduce the measurable outputs of $\mathcal{M}$ under the \textit{natural} input distribution. For deterministic settings, this is operationalized via pointwise discrepancy metrics, including MSE, RMSE, $L^2$, and their normalized variant NMSE, together with the coefficient of determination $R^2{}$~\citep{koza1994genetic,ljung1999}.
Complexity-regularized variants such as AIC~\citep{akaike2003new}, BIC~\citep{schwarz1978estimating}, MDL~\citep{RISSANEN1978465}, and Mallows' $C_p$~\citep{mallows1973,dst2009} penalize explanation complexity to prevent overfitting. When $\mathcal{M}$ is stochastic, distributional similarity is measured via measures like KL divergence, the symmetric JS divergence, or the kernel-based MMD~\citep{JMLR:v13:gretton12a}; HSIC~\citep{gretton05} can further test whether residuals of the explanation are independent of the input, probing global agreement.
For dynamical systems, validity must hold over entire trajectories: we evaluate trajectory MSE, dynamic time warping [DTW; ~\citealp{sakoe2003dynamic}], temporal autocorrelation matching, and spectral analysis~\citep{percival1993spectral}. Finally, we also consider symbolic regression methods like the SINDy evaluator~\citep{brunton_discovering_2016}, which scores an explanation by the residual of its governing equations on the observed state derivatives.
All observational criteria share a fundamental limitation called \emph{equifinality}~\citep{pmlr-v185-valogianni22a,collins2024}: distinct mechanisms can be observationally indistinguishable, yielding explanations that often do not generalize~\citep{Ghorbani_Abid_Zou_2019, Kindermans2019, meloux2025mechanistic}.\looseness=-1

\paragraph{Functional validity.}
Functional criteria require that $\mathcal{E}$ additionally reproduce the \emph{input--output response profile} of $\mathcal{M}$. \emph{Variance decomposition} via ANOVA or Sobol sensitivity indices~\citep{sobol2001global} decomposes output variance into per-input contributions and interactions, with validity requiring agreement between the indices of the high-level model $\mathcal{E}$ and the low-level model $\mathcal{M}$. At the local level, the infidelity metric~\citep{yeh2019fidelity}, adapted to the two-model setting, measures whether $\mathcal{E}$ and $\mathcal{M}$ agree on their per-input attribution vectors. This replaces the original single-model attribution formulation with a two-model sensitivity comparison, using $\mathcal{E}$ itself as the local linear approximation. LIME~\citep{Ribeiro_lime} and SHAP~\citep{NIPS2017_8a20a862} can be understood as proxies for this notion. Globally, relational fidelity ~\citep{collins2024} requires consistency of output differences across all input pairs, ensuring that $\mathcal{E}$ preserves the functional geometry of $\mathcal{M}$. However, functional criteria still remain agnostic to internal mechanisms.

\paragraph{Information-theoretic and representational validity.}
A third class of criteria requires that $\mathcal{E}$ reproduce the representational structure of $\mathcal{M}$. This can be measured by representational similarity analysis~\citep{kriegeskorte2008representational}. 
Similarly, \emph{probing accuracy}~\citep{alain2016understanding, pimentel-etal-2020-information} measures whether the variables posited by $\mathcal{E}$ are decodable from the internal states of $\mathcal{M}$. 
The \emph{information bottleneck} (IB) Lagrangian ~\citep{tishby2000informationbottleneckmethod} characterizes valid representations as those achieving an optimal compression--relevance trade-off.
\emph{Complexity shift}~\citep{ZENIL20191160} measures whether $\mathcal{E}$ induces comparable transformations under perturbation, assessed via Kolmogorov complexity. These criteria capture \emph{what} information is present, but not \emph{how} it is causally transformed~\citep{elazar-etal-2021-amnesic,lasri-etal-2022-probing,exp_unders}.

\paragraph{Causal and interventional validity.}
The strictest criteria require $\mathcal{E}$ to correctly reproduce $\mathcal{M}$'s behavior under \emph{active interventions}. \emph{Symbion}~\citep{gritti2019symbion} is a binary-analysis technique that interleaves symbolic and concrete execution, using consistency between symbolic and concrete states to guide analysis through code regions that are difficult to handle symbolically. We draw inspiration from this notion of symbolic--concrete consistency to define a Symbion-inspired validity metric. The \emph{causal sensitivity index} and the related \emph{structural deviation} metric~\citep{katende2025causaloperatordiscoverypartial} measure how much the interventional consistency score between $\mathcal{E}$ and $\mathcal{M}$ changes when each parameter of $\mathcal{E}$ is zeroed-out or perturbed by a fixed fraction, respectively.
Mechanistic interpretability~\citep{olah2020zoom} seeks to recover the sparse \emph{circuits}~\citep{NEURIPS2020_92650b2e,meng2022locating} implementing a given behavior, with validity defined by output agreement under computational interventions.  Finally, $\epsilon$-machines~\citep{PhysRevLett.63.105,shalizi_computational_2001} define the canonical valid explanation as the minimal sufficient statistic for predicting the next steps of a dynamical system, $\mathcal{M}$, potentially under interventions. The \emph{dynamic causal consistency} (DCC) criterion~\citep{nc-mcm} offers a practical implementation of this idea.\looseness=-1 

\subsection{Zooming in on Causal Abstraction}
\label{ssec:causal_abs}
Causal abstraction aims to provide a mathematical foundation for determining when high-level models can safely be considered as causal proxies for low-level ones~\citep{pmlr-v51-chalupka16,Rubensteinetal17,Beckers_Halpern_2019}. It is the central conceptual framework in this work. We explore its formal definitions and the ongoing debates regarding its precise framing.

\paragraph{Overview.}
We denote $\mathcal{M}$ as the micro- (or low-level) model and $\mathcal{E}$ as the macro- (or high-level) model. $\mathcal{E}$ is now also defined as an SCM. An abstraction between these models is defined by a mapping $\tau: \mathcal{R}^{\mathcal{M}} \to \mathcal{R}^{\mathcal{E}}$ which specifies how low-level computations in $\mathcal{M}$ translate to high-level ones in $\mathcal{E}$. \citet{zennaro2022abstraction} reviews properties of existing definitions of valid abstractions and their evolution \citep{Rubensteinetal17, Beckers_Halpern_2019, pmlr-v115-beckers20a, rischel2020category, rischel2021compositional, otsuka2022equivalence}. Two key categories of properties are relevant for us: \textbf{formal structural properties} of the mapping itself and \textbf{consistency properties}, which ensure that reasoning with either the micro- or the macro-model leads to causally consistent outputs.

\paragraph{Structural Properties.}
Structural properties define constraints on the abstraction mapping $\tau$, determining which types of abstraction maps are \textit{permissible}. Minimal but sufficient properties to model the complex systems we consider are as follows:

\textbf{1. Surjective disjoint coarse-graining of variables}: We take the coarse-graining map to be $a : \mathcal{V}_{\mathcal{M}} \to \mathcal{V}_{\mathcal{E}} \cup \{\Phi\}$, surjective onto $\mathcal{V}_{\mathcal{E}}$, where $\Phi$ is a designated dummy sentinel that collects all \textit{unmapped} micro-variables; the (possibly empty) unmapped set is $a^{-1}(\Phi)$. The map assigns each micro-variable to exactly one macro-variable or to $\Phi$, and surjectivity guarantees that every substantive macro-variable is grounded in at least one micro-variable. This corresponds to the \textbf{constructive} setting~\citep{Beckers_Halpern_2019}, which is widely used \citep{zennaro2023jointly, kekic2023targeted, zhu2024unsupervised} and that we adopt throughout. The dummy $\Phi$ carries no structural equation and no value map; it plays an operational role only during sampling, where intervening on it represents interventions on the unmapped micro-variables.

\textbf{2. Surjective value mapping}: The coarse-graining map only associates variables (structural). We must also specify how micro-states map to macro-states (contents). This is achieved via a set of \textbf{value maps} $\tau_X: \mathcal{R}^{\mathcal{M}}_{a^{-1}(X)} \to \mathcal{R}^{\mathcal{E}}_X$, which determines, for each macro-variable $X$, how its state arises from the states of its associated micro-variables. We require each value map to be surjective to guarantee that every possible macro-state can be realized from some micro-states.

Finally, the abstraction map $\tau$ consists of a surjective (disjoint) coarse-graining map $a$, the set of all surjective value maps ($\tau_X$), and the noise map $\tau_u$.

\paragraph{Consistency Property.}
For a given macro-intervention $\nu = (X\!=\!x)$, applied as $\mathrm{do}(\nu)$, multiple compatible micro-level interventions $\mu = (a^{-1}(X) = \tau_X^{-1}(x))$ may exist, applied as $\mathrm{do}(\mu)$, since $\tau_X$ is not typically injective (multiple micro-states may realize the same macro-state). Let $\mu$ be one such micro-intervention. We overload $\tau$ to the induced map on interventions, lifting targets via $a$ and values via $\tau_X$.Then, the consistency property requires that $\tau$'s noise-variable map $\tau_u: \mathcal{R}^{\mathcal{M}}_{\mathcal{U}} \to \mathcal{R}^{\mathcal{E}}_{\mathcal{U}}$ satisfy, for every compatible pair of interventions $(\nu, \mu)$ and every $u \in \mathcal{R}_{\mathcal{U}_\mathcal{M}}$, $\mathcal{E}(\mathcal{V}_\mathcal{E} | \tau_u(u), \mathrm{do}(\nu)) = \tau(\mathcal{M}\left(\mathcal{V}_\mathcal{M} | u, \mathrm{do}(\mu)\right))$. This pointwise counterfactual consistency requires both models to coincide for every exogenous  realization. When $\mathcal{E}$ is deterministic, this equation reduces to $\mathcal{E}(\mathcal{V}_\mathcal{E} \mid \mathrm{do}(\nu)) = \tau(\mathcal{M}(\mathcal{V}_\mathcal{M} \mid u, \mathrm{do}(\mu)))$, which is the primary criterion implemented in this work. This is typically visualized via a commuting diagram~\citep{Beckers_Halpern_2019}:
\begin{center}
\begin{tikzcd}
\mathcal{E}(X) \arrow[r, "\mathrm{do}(\nu)"] &  \mathcal{E}(Y)\\
\mathcal{M}(a^{-1}(X)) \arrow[r, "\mathrm{do}(\mu)"] \arrow[u, "\tau_{X}"]& \mathcal{M}(a^{-1}(Y)) \arrow[u, "\tau_Y"]
\end{tikzcd}
\end{center}

\paragraph{Measuring Consistency Errors.}
From the consistency requirement, it is possible to derive an error measure, quantifying how far a macro-model (and its associated abstraction map) is from being a valid causal abstraction of a given micro-model. A general formulation of such a measure is given by:
\[
    E_{\tau}(\nu, \mu) = D\!\left(\mathcal{E}(\mathcal{Y} \mid \mathrm{do}(\nu)),\;\tau_\mathcal{Y}\!\left(\mathcal{M}(a^{-1}(\mathcal{Y}) \mid \mathrm{do}(\mu))\right)\right),
\]
where $\mathcal{Y} \subseteq \mathcal{V}_\mathcal{E}$ is a set of macro-variables (scored jointly),
and $D$ is a pointwise dissimilarity metric (e.g., MSE) in the deterministic setting, or a distributional divergence for stochastic models. The function $E_{\tau}(\nu, \mu)$ quantifies the degree of consistency violation for a given macro-intervention $\nu$ and one associated micro-intervention $\mu$.
Different approaches have been proposed to define a \emph{global} consistency error: \citet{zennaro2023jointly} adopts the worst-case error over $(\nu, \mu)$, defining $E_{\tau} = \sup_{\nu, \mu} E_{\tau}(\nu, \mu)$ and using the Jensen-Shannon divergence as the choice for $D$. Then, \citet{zhu2024unsupervised} and \citet{kekic2023targeted} employ the Kullback-Leibler divergence and compute the expectation over interventions, facilitating sample-based approximations. A different approach is interchange intervention accuracy (IIA;~\citealp{pmlr-v162-geiger22a}), developed especially for neural networks. IIA requires a variable correspondence (a coarse-graining) and a mechanism to copy the micro-state realizing a source macro-value into a target run, but it does not require a full surjective value map with invertible preimages, and it checks consistency only at output variables. Even perfect IIA therefore does not certify causal abstraction in the general sense, since intermediate variables may still violate consistency~\citep{meloux2025dead,sutter2025nonlinearrepresentationdilemmacausal}.

\paragraph{Faithfulness.}
The standard definition of causal abstraction does not require all low-level variables in $\mathcal{V}_{\mathcal{M}}$ to be mapped: $a^{-1}(\Phi)$ may be non-empty. This is natural for AI interpretability, where explanations often posit that only a small subset of a neural network is relevant to a behavior \citep{olah2020zoom,NEURIPS2023_34e1dbe9}, and similar sparsity assumptions appear in biological systems~\citep{rochefort2009sparsification,meunier2010modular}. Yet, declaring excluded components causally irrelevant is a strong empirical claim, that must be tested. In circuit discovery, this is captured by \textbf{faithfulness}: unmapped neurons should be intervened on to verify that they do not influence the proposed circuit computation \citep{hanna2024have}. Existing operationalizations of causal abstraction largely take the relevant variable set as modeler-provided, rather than testing whether excluded variables are genuinely inert~\citep{zennaro2023jointly,kekic2023targeted,zhu2024unsupervised}. Faithfulness checks reject explanations that incorrectly assume away causal influence from omitted components.

%% file: sections/03_cae.tex
\section{The Causal Abstraction Error}
\label{sec:cae}

Building upon existing causal alignment error measures \citep{zennaro2023jointly, kekic2023targeted, zhu2024unsupervised}, we define a new metric that explicitly captures faithfulness in causal abstraction.
We first define what we term an instance-level abstraction error (\textnormal{\scriptsize IAE}):
\begin{align}
    \textsc{iae}(\mathcal{M} \overset{\tau}{\to} \mathcal{E},\mathcal{Y} \mid u,\, \mathrm{do}(\mu,\nu))
 = D\!\left(\mathcal{E}(\mathcal{Y}\mid \tau_u(u),\, \mathrm{do}(\nu)),\;
   \tau_{\mathcal{Y}}(\mathcal{M}(a^{-1}(\mathcal{Y})\mid u,\, \mathrm{do}(\mu)))\right)
\end{align}
where $\mathcal{Y} \subseteq \mathcal{V}_\mathcal{E}$ is a set of macro-variables (scored jointly), and $D$ is a dissimilarity over $\mathcal{R}_\mathcal{Y}^\mathcal{E}$.
In words, this metric computes whether, given a controlled exogenous value $u$ and intervention $\mathrm{do}(\mu, \nu)$, the two SCMs produce compatible outputs for variables $\mathcal{Y}$ at a counterfactual level.
However, this is an instance-level error.
In practice, we are interested in how our high-level explanation performs across several exogenous values and interventions.
We thus define two variants of \textbf{causal abstraction error} (\cae{}) that differ in how we aggregate these values:
\begin{subequations}
\begin{align}
\cae_{\downarrow}^{\text{Agg}}(\mathcal{M} \overset{\tau}{\to} \mathcal{E}, \mathcal{Y}) &= \text{Agg}_{\nu \sim P_I,\; \mu \sim P_{\tau^{-1}\!(\nu)},\; u \sim P_{\mathcal{U}_\mathcal{M}}}\!\!\left[ \textsc{iae}\left(\mathcal{M} \overset{\tau}{\to} \mathcal{E},\mathcal{Y} \mid u,\, \mathrm{do}(\mu,\nu)\right) \right]\\
\cae_{\uparrow}^{\text{Agg}}(\mathcal{M} \overset{\tau}{\to} \mathcal{E}, \mathcal{Y}) &= \text{Agg}_{\mu \sim P_I,\; u \sim P_{\mathcal{U}_\mathcal{M}}}\left[ \textsc{iae}\left(\mathcal{M} \overset{\tau}{\to} \mathcal{E},\mathcal{Y} \mid u,\, \mathrm{do}(\mu,\nu)\right) \right]
\label{eq:cae}
\end{align}
\end{subequations}
where $P_I$ denotes a distribution over interventions on the set of micro- or macro-variables (in the $\uparrow$ and the $\downarrow$ variant, respectively);
and $P_{\tau^{-1}(\nu)}$ denotes a distribution over micro-interventions $\mu$ compatible with $\nu$, i.e., for which $\nu = \tau(\mu)$.
\cae{} thus aggregates $\textsc{iae}$ values using a function $\text{Agg}$ which operates over the indexed variables $(\nu,\mu,u)$ through distributions $P_I$, $P_{\tau^{-1}\!(\nu)}$, and $P_{\mathcal{U}_\mathcal{M}}$; this function can be, e.g., the expectation, supremum, infimum, or combinations thereof.
Notably, the aggregation over exogenous variables $u \sim P_{\mathcal{U}_\mathcal{M}}$ could be folded with the dissimilarity metric $D$ and replaced by a distributional divergence instead (e.g., JSD, KL, MMD) for stochastic models, computing e.g. $\mathrm{KL}\!\left(\mathcal{E}(\mathcal{Y} \mid \mathrm{do}(\tau(\mu))),\; \tau_{\mathcal{Y}}\!\left(\mathcal{M}(a^{-1}(\mathcal{Y}) \mid \mathrm{do}(\mu))\right)\right)$, where $u$ is marginalized out; in that case, $\cae$ would measure causal consistency at the intervention (instead of counterfactual) level.

In the top-down variant ($\downarrow$), macro-interventions are grounded to compatible micro-interventions via $\tau_X^{-1}$; in the bottom-up one ($\uparrow$), micro-interventions are drawn directly and lifted to the macro level via $\tau_X$. Note that these two approaches are equivalent in their zero set; however, their non-zero estimates may differ due to sampling decisions.
Both arguments of $D$ lie in the macro-variable range $\mathcal{R}_\mathcal{Y}$, avoiding ill-posedness from $\tau_{\mathcal{Y}}$ being non-injective. 
Faithfulness is integrated in both variants: $P_I$ can select the dummy macro-variable $\Phi$ as an intervention target (or, in the $\uparrow$ variant, it can target the unmapped micro-variables $a^{-1}(\Phi)$), which the proposed explanation claims are causally inert. We also define non-faithful variants, denoted by the subscript $_\text{NF}$: in those, $P_I$ excludes $\Phi$.

\paragraph{Relation to \citet{zennaro_quantifying_2023}.}
That work defines four related metrics; ISIL and IC are the closest to \caedown{} and \caeup{} respectively. Beyond the faithfulness consideration, three differences are worth noting. First, their metrics aggregate by supremum; ours use expectation, which is easier to estimate by sampling. Second, their framework is restricted to finite variable domains, whereas \cae{} handles continuous value maps directly. Third, their stochastic treatment prevents compositionality of valid abstractions from being preserved, which is not the case for the \cae{} (see Appendix~\ref{app:theoretical}).

\subsection{Practical Estimation}
\label{ssec:estimation}

In all of our experiments, we use the expectation $\mathbb{E}$, approximated by the empirical mean, as the aggregation operator $\text{Agg}$ in \cae{}'s definition.
Both \cae{} variants thus use expectations amenable to Monte Carlo estimation.
The formalism samples only exogenous variables $u \in \mathcal R_{\mathcal U}^{\mathcal M}$: the macro-level exogenous variables enter only through the mapping $\tau_u$. In our systems, $\mathcal E$ is deterministic, making $\tau_u$ vacuous (the identity, as there is no macro-exogenous noise to map), so each Monte Carlo sample is determined solely by the draw of the intervention $(\nu,\mu)$ together with any sampled $u \in \mathcal R_{\mathcal U}^{\mathcal M}$.
For models $\mathcal M$ that are internally stochastic (e.g., the Lennard--Jones gas), additional simulator randomness is generated independently for each model evaluation. Since our implementation does not specify a coupling map for this internal randomness across the micro- and macro-level models, the resulting comparison is distributional rather than pointwise. Accordingly, these experiments instantiate the interventional interpretation of \cae{} (the folded-divergence form) rather than a counterfactual interpretation, which would require a well-defined coupling $\tau_u$ that identifies corresponding exogenous realizations across the two levels. We return to this limitation in Section~\ref{ssec:limitations}.

\paragraph{Sampling macro-interventions ($P_I$ in \caedown{}).}
We define $P_I$ as the following hierarchical distribution. First draw a cardinality $k \sim \mathrm{Uniform}\{1,\ldots,\min(|\mathcal{V}_\mathcal{E}|, \text{max\_interventions})\}$; then draw the target set $\mathcal{S}$ uniformly among all size-$k$ subsets of $\mathcal{V}_\mathcal{E}$; finally draw intervention values $\mathbf{x}_\mathcal{S}$ with each $x_V$ ($V\in\mathcal{S}$) i.i.d.\ uniform on the domain $\mathcal{R}_V$. We draw $n$ such interventions $\{\nu^{(i)}\}_{i=1}^n \sim P_I$.

\paragraph{Sampling compatible micro-interventions ($P_{\tau^{-1}(\nu)}$ in \caedown{}).}
For each macro-intervention $\nu$ on variables $\mathcal{S}$, we sample a compatible micro-intervention by applying the inverse value map to each intervened variable: for $V \in \mathcal{S}$ with intervened value $x_V$, we compute the preimage $\tau_V^{-1}(x_V)$ and sample a micro-state from it, which corresponds to the construction from~\cite{zennaro2023jointly}.

\paragraph{Sampling micro-interventions ($P_I$ in \caeup{}).}
We draw $n$ micro-interventions $\{\mu^{(i)}\}_{i=1}^n \sim P_I$ directly, using the same multi-node scheme as above but operating over micro-variable domains. The corresponding macro-intervention is obtained by applying $\tau_X$ to the intervened micro-state.

\paragraph{Testing faithfulness.} We test faithfulness by augmenting the macro-intervention pool with the dummy $\Phi$. At each sample, $\Phi$ is included in the selected intervention targets with a custom probability. When selected, we first ground the substantive interventions to micro-values as usual, then overwrite the micro-value of each unmapped variable with sampled noise before executing $\mathcal{M}$ (in the $\uparrow$ variant, $a^{-1}(\Phi)$ is instead included directly in the micro-intervention pool): Gaussian noise for continuous variables, uniform $\{-1,0,1\}$ perturbations for integers, and random resampling for Booleans. Since $\Phi$ is a sentinel, it induces no intervention on the macro side: any shift in the abstracted output relative to the unperturbed run reflects a causal influence of $a^{-1}(\Phi)$ and increases the consistency error.

%% file: sections/04_benchmark.tex
\section{A Benchmark of Simulated Complex Systems}
\label{sec:benchmark}
A central obstacle in evaluating explanation-validity metrics is the lack of naturally occurring systems paired with objective ground-truth explanations. Human judgment is an inadequate substitute: an explanation can be \emph{plausible} (convincing to a human) without being \emph{faithful} (an accurate account of the underlying mechanism), and users often prefer the former even when it is incomplete or misleading~\citep{jacovi-goldberg-2020-towards,ancona2019gradient,herman2017promise}.
Further, existing benchmarks typically evaluate desirable \emph{properties} of explanation methods (sparsity, stability, localization, or causal robustness) rather than whether an explanation recovers a \emph{known} valid high-level causal account. For instance, SAEBench \citep{karvonen2025saebenchcomprehensivebenchmarksparse} scores sparse autoencoders on feature disentanglement and reconstruction quality, and RAVEL \citep{huang-etal-2024-ravel} measures how cleanly interventions isolate targeted attributes; both ask whether a method behaves well by some proxy, not whether its output matches a ground-truth mechanism. The distinction matters because a method can score well on stability or sparsity while still recovering the wrong causal structure.\looseness=-1

We address both issues above by constructing a benchmark of idealized complex systems whose high-level explanations are known and analytically justified at the relevant abstraction level. This benchmark enables us to verify two desired properties for a target metric: it should assign near-zero error to correct explanations, and it should detect controlled failures produced by invalid explanations.
A key contribution of this work is the manual alignment of these heterogeneous systems with their high-level explanations integrated within a common framework. The benchmark spans many systems chosen to cover different abstraction types, yet all systems admit clean causal-abstraction representations and, more strongly, \emph{constructive} causal abstractions. This provides a reusable resource and a nontrivial empirical demonstration of the breadth of the causal abstraction program. For each system, we provide the high-level model, low-level implementation, abstraction map, and manually crafted invalid contrastive scenarios. Full details are given in Appendix~\ref{app:systems}; Table~\ref{tab:systems} summarizes the systems.\looseness=-1

\begin{table}[ht]
\centering
\caption{
\textbf{Benchmark systems.}
Each row is a manually aligned ground-truth abstraction pair $(\mathcal{M},\mathcal{E},\tau)$.
The benchmark spans discrete (D) and continuous (C) state spaces, static and dynamical regimes, and local, spatial, and global constructive maps.
\emph{Structure} describes how each macro-observable is formed from the micro-state; for ``aggregation'' and ``observables'', this is performed within $\mathcal{M}$'s measurement, with $\tau$ acting as the identity on the resulting macro-variables.
Rows shaded in light blue are dynamical systems.
For every pair, we also provide invalid contrastive conditions by perturbing the model, the map, or the validity regime; full details are in Appendix~\ref{app:systems}.
}
\label{tab:systems}
\vspace{3pt}
\footnotesize
\setlength{\tabcolsep}{3.2pt}
\renewcommand{\arraystretch}{1.06}
\begin{tabularx}{\linewidth}{@{}lYYC{1.45cm}C{2.25cm}@{}}
\toprule
\textbf{System} &
\textbf{Low level $\mathcal{M}$} &
\textbf{High level $\mathcal{E}$} &
\multicolumn{2}{c}{\textbf{Constructive abstraction map $\tau$}} \\
\cmidrule(lr){4-5}
& & & \textbf{Type} & \textbf{Structure} \\
\midrule

\rowcolor{dynrow}
Gas
& Lennard--Jones particles
& Ideal gas / VdW law
& C $\to$ C
& global / aggregation \\

\rowcolor{dynrow}
Predator--prey
& Agent-based model
& Lotka--Volterra ODE
& D $\to$ C
& global / aggregation \\

\rowcolor{dynrow}
Heat equation 1D
& Brownian particles
& 1D heat equation
& C $\to$ C
& spatial / binning \\

\rowcolor{dynrow}
Heat equation 2D
& Phonon/lattice simulation
& 2D heat equation
& C $\to$ C
& spatial / smoothing \\

\rowcolor{dynrow}
Ising model
& MD/Monte-Carlo simulator
& Nearest-neighbor Ising model
& C $\to$ C
& global / observables \\

Tracr transformer
& Compiled transformer states
& Symbolic sort-rank program
& D $\to$ D
& local / decoding \\

Gene regulatory network
& Multi-valued GRN
& Binary Wg--Fz rule
& D $\to$ D
& local / threshold \\

Logic circuit
& Wire-level Boolean circuit
& Multi-bit gate circuit
& D $\to$ D
& local / functional \\

Transistor circuit
& SPICE/MOSFET voltages
& Boolean gate network
& C $\to$ D
& local / threshold \\

MOS 6502 CPU
& Transistor-level sim.
& Gate-level sim.
& D $\to$ D
& local / functional \\

MOS 6502 CPU
& Transistor-level sim.
& ISA-level sim.
& D $\to$ D
& local / functional \\

MOS 6502 CPU
& Transistor-level sim.
& ISA-level sim.
& D $\to$ D
& local / functional \\

\bottomrule
\end{tabularx}
\end{table}

%% file: sections/05_results.tex
\section{Results}
\label{sec:results}

We evaluate metrics along three dimensions in sequence: classification performance under high sampling, discrimination of structurally invalid abstractions in controlled experiments, and sampling efficiency as a function of model size.

\label{ssec:benchmarking_metrics}
\paragraph{Benchmarking Metrics.}
We begin by asking which metrics can, in principle, separate valid from invalid abstractions when sampling is not a constraint. We evaluate all metrics in the asymptotic regime on every (valid, invalid) abstraction pair in the benchmark, with full results reported in Appendices~\ref{app:valid_scores} and~\ref{app:invalid_scores}. These results motivate two principled thresholds: after normalization, a metric should score at most 0.10 on a valid abstraction and at least 0.15 on an invalid one. Figure~\ref{fig:metric_scores} uses these thresholds to assess the classification performance of each metric across the benchmark.\tiago{A nice figure could be to just have: y-score the quality metric (e.g., CAE) and x-axis 0, 1 (valid/invalid) for all matrics. Or a figure with x-axis: quality metric on valid explanation, y-axis quality metric on invalid explanation + grey identity line. If a method is always above identity, it identifies invalid explanations as worse than valid ones (within a specific setting); if a method is always y-axis larger than largest x-axis, it can fully discriminate good from bad explanations in general (across settings). This might also help understand figure 1?}

\begin{figure}[ht]
    \centering
    \includegraphics[width=\linewidth]{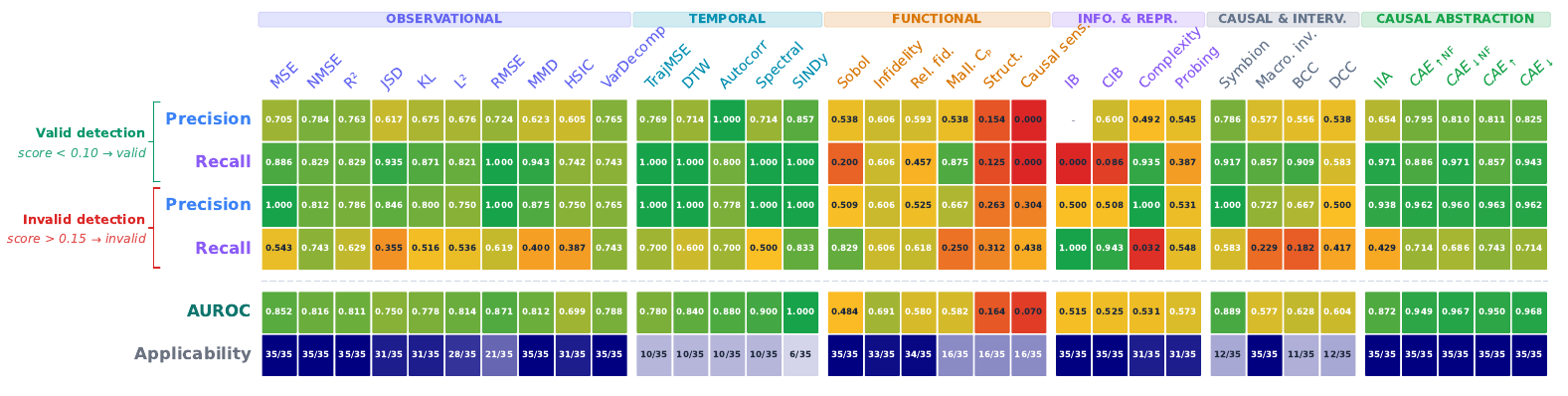}
    \caption{Precision, recall, and AUROC of each metric, computed over all valid and invalid abstractions of the benchmark (higher is better). Applicability denotes the fraction of abstractions for which each metric was able to be evaluated; by design, many metrics are only applicable to a restricted class of systems. Metric normalization, implementation, and hyperparameters are reported in Appendix \ref{app:metric_impl}.}
    \label{fig:metric_scores}
\end{figure}

Coverage (applicability) immediately eliminates a large subset of candidates. Temporal metrics (TrajMSE, DTW, Autocorr, Spectral, SINDy) are defined only for dynamical systems; Symbion requires finite discrete label spaces; $C_p$ requires a linear proxy model with a well-defined noise floor; and structural deviation and the causal sensitivity index fail to produce scores on most benchmark systems.
Among metrics with broad coverage, information-theoretic scores (IB and CIB Lagrangian) and several functional metrics, such as Sobol indices, structural deviation, and causal sensitivity indices, perform poorly on valid abstractions: they are sensitive to distributional details unrelated to causal consistency and assign high error even to valid abstractions. For invalid abstractions, some metrics attain reasonable precision, but at the cost of low recall. This combination of poor calibration and limited applicability eliminates many candidate metrics from further consideration. Conversely, causal abstraction metrics consistently achieve the highest AUROC and maximum coverage.

\label{ssec:discrimination}
\paragraph{Detection of Controlled Failure Cases.}
Classification performance under unlimited sampling does not capture whether a metric can detect the specific structural failures that motivate causal abstraction. We therefore turn to six controlled experiments, each designed to expose a subtle violation that is observationally indistinguishable from the valid condition or manifests only under targeted interventions. The six settings are illustrated in Figure~\ref{fig:controlled_graphs} and described in more detail in Appendix~\ref{app:controlled_exp}.
Figure~\ref{fig:invalid_scores_controlled} then reports, for each experiment and metric, whether a Mann--Whitney $U$ test over 100 runs reveals a statistically significant difference between the valid and invalid condition.

\input{sections/10_control_illustration}

\begin{figure}[ht]
    \centering
    \includegraphics[width=\linewidth]{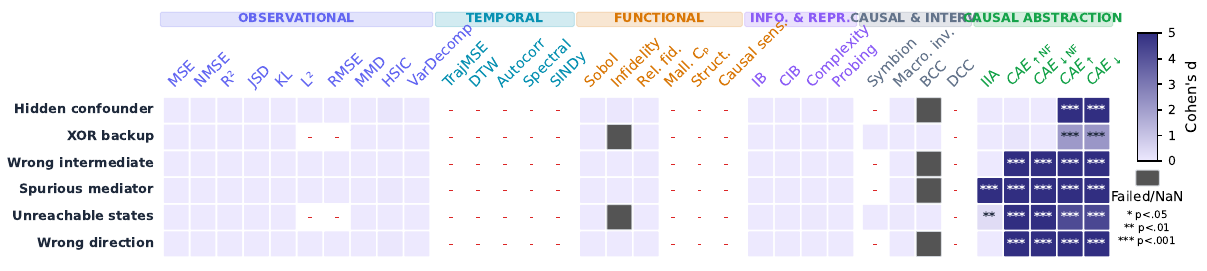}
    \caption{\textbf{Controlled failure modes.} Each cell reports whether a metric significantly separates the valid and invalid abstraction over 100 runs using a Mann--Whitney $U$ test. \cae{} is the only family that detects all six targeted invalidities.}
    \label{fig:invalid_scores_controlled}
\end{figure}

The results show that among the metrics that do apply, none of the observational, functional, or information-theoretic criteria achieve consistent discrimination across the six experiments. This is expected:  each invalid abstraction is observationally indistinguishable from the valid one or fails only under specific interventional conditions, which are precisely the failure modes that causal abstraction was designed to detect. Within the causal abstraction family, results are more nuanced. IIA only succeeds in detecting the violations in Experiments~4 and 5. Both $\cae_\text{NF}$ variants fail in Experiment~1 and 2; only the \cae{} variants successfully detect the violations.

Together, the benchmark results and controlled experiments narrow the candidate set to four metrics. IIA is excluded because, despite broad coverage, it fails to detect the violations in Experiments~3 and 6 where other causal abstraction metrics succeed. The remaining candidates are \caedown{} and \caeup{}, which we carry forward as the primary metrics. Their faithfulness-free counterparts, \caedownnf{} and \caeupnf{}, are retained as ablations to isolate the contribution of faithfulness testing.

\label{ssec:power}
\paragraph{Statistical Power.}
Having narrowed the candidate set to four metrics, we ask how many interventions each requires to reliably detect an invalid abstraction. Convergence results are reported in Appendix~\ref{app:power_curves}. Figure~\ref{fig:power_curves} reports statistical power across selected systems and controlled experiments.

\begin{figure}
    \centering
    \includegraphics[width=\linewidth]{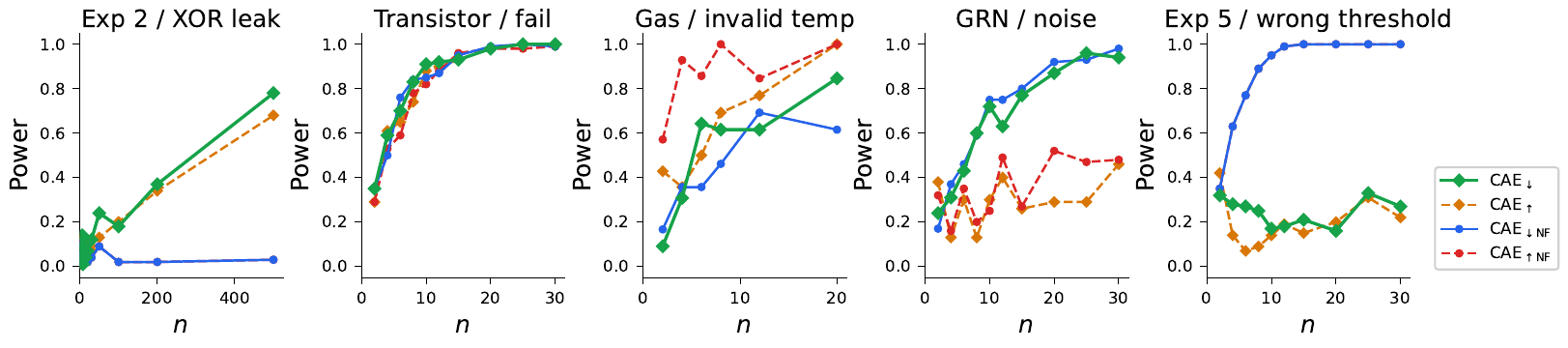}
    \caption{How quickly do metrics detect bad abstractions over good ones? For different numbers of sampled interventions $n$, we compute each metric 100 times on selected valid and invalid abstractions. We apply a Mann--Whitney $U$ test to the two sets of scores (valid and invalid) and report the fraction of runs for which the results differ significantly ($p<0.05$). Note that both NF variants overlap in the leftmost and rightmost plots. Full results are given in Appendix~\ref{app:results}.}
    \label{fig:power_curves}
\end{figure}

Statistical power converges to nearly 100\% within 30 interventions for most systems and all four metrics. Two exceptions stand out. For the gas simulation, power increases more slowly, consistent with the higher variance noted above. Experiment~2 (XOR leak) illustrates the usefulness of faithfulness testing: metrics that include it detect the invalid abstraction as $n$ grows, while their faithfulness-free counterparts consistently fail. The noise condition on the GRN system illustrates the value of top-down sampling: \caedown{} and \caedownnf{} can reliably detect the invalid abstraction after 25--30 interventions, while bottom-up metrics converge more slowly. For Experiment~5 (unreachable intermediate states), the picture reverses: faithfulness-aware metrics converge more slowly and exhibit lower power than plain \caenf{}. \caedownnf{} already probes unreachable states by sampling from the full abstract label domain; faithfulness augmentation additionally injects perturbations into unmapped $\Phi$ variables, which occasionally push $\mathcal{M}$ into the same micro-region as the unreachable states, inflating the valid-condition score and shrinking the discrimination gap. Faithfulness testing therefore does not miss the invalid abstraction here, but reduces specificity on the valid side.


\begin{wrapfigure}{r}{0.42\linewidth}
    \vspace{-0.75em}
    \centering
    \includegraphics[width=\linewidth]{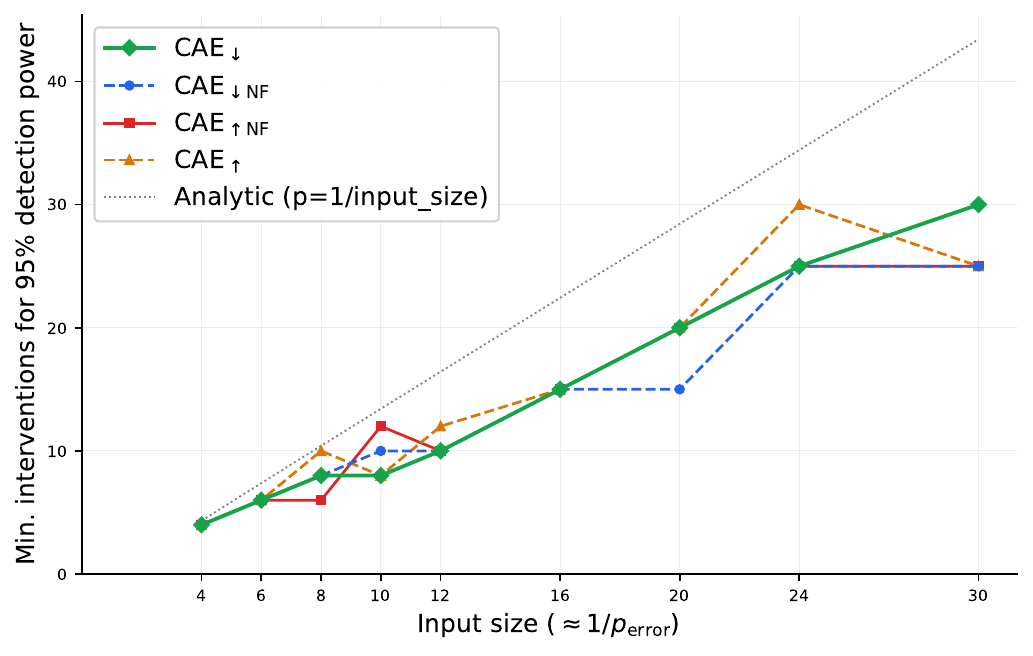}
    \caption{Minimum number of interventions required to reach 95\% detection power as a function of $|\mathcal{\mathcal{V}_\mathcal{E}}|$, measured on Tracr sort-rank programs of increasing sequence length (100 runs per length). 
    See details in Appendix~\ref{app:tracr_scaling}.}
    \label{fig:tracr_scaling}
    \vspace{-2.5em}
\end{wrapfigure} 

\label{ssec:scaling_laws}
\paragraph{Scaling.}
Power curves on fixed systems do not reveal how detection cost scales as the high-level model $\mathcal{E}$ grows. To answer this question, we construct a controlled scaling experiment using Tracr~\citep{lindner2023tracr}, a compiler that translates RASP programs into exact transformer weights. We compile a sort-rank program for each sequence length $\ell \in \{2, 3, 4, 5, 6, 8, 10, 12, 15\}$, yielding transformers of increasing size that compute the rank of each token in the input sequence. The valid high-level model ($\mathcal{E}$) encodes the correct formula $\mathrm{rank}_i = |\{j \neq i : \mathrm{token}_j < \mathrm{token}_i\}|$ for all positions $i$, while the wrong model encodes a specific incorrect hypothesis: the minimum-valued token at position 0 always receives the maximum rank rather than rank 0. This error triggers exactly when $\mathrm{token}_0$ is the minimum element of the sequence, an event with probability $1/\ell$ under uniform sampling of distinct tokens. The expected number of interventions to observe at least one triggering event therefore grows linearly with $\ell$, and analytic calculation gives $n_{95\%} \approx 2.5\ell$ for 95\% power.

Figure~\ref{fig:tracr_scaling} shows the empirically measured required $n$ for \cae{} and \caenf{} as a function of $|\mathcal{E}| = 2\ell$ (token inputs plus rank outputs). All four metrics follow the predicted linear trend closely. This confirms that the sampling budget is governed primarily by the rarity of the triggering condition in the abstract domain, not by the internal complexity of the transformer. In practice, a wrong explanation that diverges from the correct one on a fraction $p$ of inputs requires $O(1/p)$ interventions to detect with fixed confidence: rare but systematic errors in proposed explanations remain detectable at a cost proportional to their rarity, independent of model depth or width.

%% file: sections/10_control_illustration.tex
\tikzset{
  hn/.style  = {circle,draw,thick,minimum size=5mm,font=\scriptsize,inner sep=0.5pt},
  ln/.style  = {circle,draw,minimum size=5mm,font=\scriptsize,inner sep=0.5pt},
  pn/.style  = {circle,draw,dashed,gray,minimum size=5mm,font=\scriptsize,inner sep=0.5pt},
  ae/.style  = {->,gray!70,dashed,shorten >=1.5pt,shorten <=1.5pt},
  iv/.style  = {red,thick},
  rm/.style  = {->,gray!25,densely dashed,thin,shorten >=1.5pt,shorten <=1.5pt},
  lb/.style  = {font=\tiny\itshape},
  tt/.style  = {font=\footnotesize\bfseries},
  nz/.style  = {font=\tiny,text=gray,inner sep=1pt},
  ne/.style  = {->,gray!55,shorten >=1pt,shorten <=1pt},
  nzi/.style = {font=\tiny,text=red!75,inner sep=1pt},
  nei/.style = {->,red!70,shorten >=1pt,shorten <=1pt},
}

\begin{figure*}[ht]
\centering
\resizebox{\textwidth}{!}{%
\begin{tikzpicture}

\begin{scope}[xshift=0mm, yshift=0mm]
  \node[tt] at (13mm,8mm) {(1) Hidden confounder};
  \node[hn] (H1X1) at (0,0)    {$X_1$};
  \node[hn] (H1Y)  at (13mm,0) {$Y$};
  \node[hn] (H1X2) at (26mm,0) {$X_2$};
  \draw[->,thick] (H1X1)--(H1Y);
  \draw[->,thick] (H1X2)--(H1Y);
  \node[lb,left=1pt of H1X1] {$\mathcal{E}$:};
  \node[ln] (L1X1) at (0,-13mm)    {$x_1$};
  \node[ln] (L1Y)  at (13mm,-13mm) {$y$};
  \node[ln] (L1X2) at (26mm,-13mm) {$x_2$};
  \draw[->] (L1X1)--(L1Y);
  \draw[->] (L1X2)--(L1Y);
  \node[lb,left=1pt of L1X1] {$\mathcal{M}$:};
  \draw[ae] (L1X1)--(H1X1);
  \draw[ae] (L1Y)--(H1Y);
  \draw[ae] (L1X2)--(H1X2);
  \node[circle,draw=red,dashed,minimum size=5mm,font=\scriptsize,
        inner sep=0.5pt,text=red] (U1) at (20mm,-6mm) {$r$};
  \draw[->,iv] (U1)--(L1Y);
\end{scope}

\begin{scope}[xshift=38mm, yshift=0mm]
  \node[tt] at (15mm,8mm) {(2) XOR backup path};
  \node[hn] (H2X) at (0,0)    {$X$};
  \node[hn] (H2M) at (13mm,0) {$M$};
  \node[hn] (H2Y) at (30mm,0) {$Y$};
  \draw[->,thick] (H2X)--(H2M);
  \draw[->,thick] (H2M)--(H2Y);
  \node[lb,left=1pt of H2X] {$\mathcal{E}$:};
  \node[ln] (L2X)  at (0,-13mm)    {$x$};
  \node[ln] (L2a)  at (13mm,-13mm) {$m$};
  \node[ln] (L2Y)  at (30mm,-13mm) {$y$};
  \node[pn] (L2b1) at (7mm,-5mm)   {$a$};
  \node[pn] (L2b2) at (7mm,-21mm)  {$b$};
  \draw[->] (L2X)--(L2a);
  \draw[->] (L2a)--(L2Y);
  \draw[->] (L2X)--(L2b1);
  \draw[->] (L2X)--(L2b2);
  \node[lb,left=1pt of L2X] {$\mathcal{M}$:};
  \draw[ae] (L2X)--(H2X);
  \draw[ae] (L2a)--(H2M);
  \draw[ae] (L2Y)--(H2Y);
  \draw[->,iv] (L2b1) to[bend left=12]  (L2Y);
  \draw[->,iv] (L2b2) to[bend right=12] (L2Y);
  \node[lb,iv] at (24mm,-9mm) {$m\!\lor\!(a\!\oplus\!b)$};
\end{scope}

\begin{scope}[xshift=79mm, yshift=0mm]
  \node[tt] at (13mm,8mm) {(3) Wrong intermediate};
  \node[hn] (H3X) at (0,0)    {$X$};
  \node[hn] (H3M) at (13mm,0) {$M$};
  \node[hn] (H3Z) at (26mm,0) {$Y$};
  \draw[->,thick] (H3X)--(H3M);
  \draw[->,thick] (H3M)--(H3Z);
  \node[lb,left=1pt of H3X] {$\mathcal{E}$:};
  \node[lb]    at (13mm, 5mm) {$M\!=\!2X$};
  \node[lb,iv] at (13mm,-5mm) {$M\!=\!3X$};
  \node[ln] (L3X) at (0,-13mm)    {$x$};
  \node[ln] (L3M) at (13mm,-13mm) {$m$};
  \node[ln] (L3Z) at (26mm,-13mm) {$y$};
  \draw[->] (L3X)--(L3M);
  \draw[->] (L3M)--(L3Z);
  \node[lb,left=1pt of L3X] {$\mathcal{M}$:};
  \node[lb,below=2pt of L3M] {$m\!=\!2x$};
  \draw[ae] (L3X)--(H3X);
  \draw[ae] (L3M)--(H3M);
  \draw[ae] (L3Z)--(H3Z);
\end{scope}

\begin{scope}[xshift=121mm, yshift=0mm]
  \node[tt] at (13mm,8mm) {(4) Spurious mediator};
  \node[hn] (H4X) at (0,0)    {$X$};
  \node[hn] (H4Y) at (11mm,0) {$M$};
  \node[hn] (H4Z) at (26mm,0) {$Y$};
  \draw[->,thick] (H4X)--(H4Y);
  \draw[rm] (H4X) to[bend right=22] (H4Z);
  \node[lb,gray] at (13mm,-5mm) {\tiny rm.};
  \draw[->,iv] (H4Y)--(H4Z);
  \node[lb,iv] at (19mm,3mm) {\tiny chain};
  \node[lb,left=1pt of H4X] {$\mathcal{E}$:};
  \node[ln] (L4X) at (0,-13mm)    {$x$};
  \node[ln] (L4Y) at (11mm,-13mm) {$m$};
  \node[ln] (L4Z) at (26mm,-13mm) {$y$};
  \draw[->] (L4X)--(L4Y);
  \draw[->] (L4X) to[bend right=22] (L4Z);
  \node[lb,left=1pt of L4X] {$\mathcal{M}$:};
  \draw[ae] (L4X)--(H4X);
  \draw[ae] (L4Y)--(H4Y);
  \draw[ae] (L4Z)--(H4Z);
\end{scope}

\begin{scope}[xshift=159mm, yshift=0mm]
  \node[tt] at (13mm,8mm) {(5) Unreachable states};
  \node[hn] (H5X) at (0,0)    {$X$};
  \node[hn] (H5M) at (13mm,0) {$M$};
  \node[hn] (H5Z) at (26mm,0) {$Y$};
  \draw[->,thick] (H5X)--(H5M);
  \draw[->,thick] (H5M)--(H5Z);
  \node[lb,left=1pt of H5X] {$\mathcal{E}$:};
  \node[lb]    at (26mm, 5mm) {$M\!\ge\!2$};
  \node[lb,iv] at (26mm,-5mm) {$M\!\ge\!1$};
  \node[ln] (L5X) at (0,-13mm)    {$x$};
  \node[ln] (L5M) at (13mm,-13mm) {$m$};
  \node[ln] (L5Z) at (26mm,-13mm) {$y$};
  \draw[->] (L5X)--(L5M);
  \draw[->] (L5M)--(L5Z);
  \node[lb,left=1pt of L5X] {$\mathcal{M}$:};
  \node[lb,below=2pt of L5M] {\tiny $M\!=\!1$ unreach.};
  \draw[ae] (L5X)--(H5X);
  \draw[ae] (L5M)--(H5M);
  \draw[ae] (L5Z)--(H5Z);
\end{scope}

\begin{scope}[xshift=197mm, yshift=0mm]
  \node[tt] at (13mm,8mm) {(6) Wrong direction};
  \node[hn] (H6X) at (0,0)    {$X$};
  \node[hn] (H6M) at (13mm,0) {$M$};
  \node[hn] (H6Z) at (26mm,0) {$Y$};
  \draw[->,thick] (H6X)--(H6M);
  \draw[rm] (H6M)--(H6Z);
  \node[lb,gray,above] at (19mm,1mm) {\tiny rm.};
  \draw[->,iv] (H6X) to[bend right=28]
      node[lb,iv,below,sloped,pos=0.5]{\tiny fork} (H6Z);
  \node[lb,left=1pt of H6X] {$\mathcal{E}$:};
  \node[ln] (L6X) at (0,-13mm)    {$x$};
  \node[ln] (L6M) at (13mm,-13mm) {$m$};
  \node[ln] (L6Z) at (26mm,-13mm) {$y$};
  \draw[->] (L6X)--(L6M);
  \draw[->] (L6M)--(L6Z);
  \node[lb,left=1pt of L6X] {$\mathcal{M}$:};
  \draw[ae] (L6X)--(H6X);
  \draw[ae] (L6M)--(H6M);
  \draw[ae] (L6Z)--(H6Z);
\end{scope}

\end{tikzpicture}%
}
\caption{%
Causal graphs and abstraction maps for the six controlled experiments. Each panel shows $\mathcal{E}$ (top, uppercase) and $\mathcal{M}$ (bottom, lowercase). Gray dashed arrows are the abstraction map $\tau$. Red marks what changes in the invalid condition; light gray dashed arrows mark edges present in the valid condition that are removed in the invalid condition. Dashed gray circles are $\Phi$ (unmapped) variables; a dashed red circle is a $\Phi$ variable introduced only in the invalid condition. $X$ and $Y$ correspond to inputs and outputs of the models, respectively. Each variable is assumed to be perturbed by exogenous noise $u$, suppressed in the diagrams for visual clarity; refer to Appendix \ref{app:controlled_exp} for details.
}
\label{fig:controlled_graphs}
\end{figure*}

%% file: sections/06_discussion.tex
\section{Discussion}
\label{sec:discussion}

This paper asks when a high-level explanation of a low-level system should be trusted. Our benchmark turns this into an empirical stress test: given valid and invalid abstraction pairs, which metrics separate them? The main result is that observational agreement, distribution matching, representational similarity, compression, symbolic fit, and input-output sensitivity can be useful diagnostics, but not validity criteria.
Across heterogeneous systems, causal-abstraction metrics are the most reliable. However, our results also show that operationalization matters: not all causal metrics succeed and different implementations fail in different ways. The proposed \cae{} metrics, which combine multi-node interventions with built-in faithfulness tests, are the only metrics that pass all tests.
This makes \cae{} a natural objective for explanation discovery. Symbolic, gradient-based, evolutionary, and human-guided search procedures all need a criterion that rewards causally valid abstractions. In particular, future work should study adaptive intervention sampling to quickly detect invalid abstraction, \cae{}-based abstraction discovery, and community extensions of the benchmark. We release the benchmark and all metric implementations on \href{https://github.com/MelouxM/CAE}{GitHub} and \href{https://pypi.org/project/causal-abstraction-eval/}{PyPI} to support this direction.
\subsection{Limitations}
\label{ssec:limitations}


\paragraph{Validity of a given triple vs.\ discovery of good triples.}
\cae{} evaluates a \emph{specified} $(\mathcal{M},\mathcal{E},\tau)$; it does not by itself prevent gerrymandered abstraction maps that are made complex enough to pass, the abstraction analogue of the non-linear representation dilemma~\citep{sutter2025nonlinearrepresentationdilemmacausal,meloux2025dead}.
Three ingredients mitigate but do not eliminate this: constructive (disjoint) coarse-graining restricts $\tau$ to surjective variable partitions; faithfulness testing penalizes maps that silently assume away causal influence; and multi-node interventions probe interaction terms that single-variable checks miss. A complete account of \emph{which} valid abstractions are also \emph{good} explanations (parsimony, level-appropriateness) is beyond a consistency metric and remains future work.

\paragraph{Handling of exogenous noise.} All macro-models in our benchmark are deterministic, and the reported results use $\tau_u$ as the identity. The implementation optionally supports a stochastic high-level model $\mathcal{E}$ that carries an exogenous distribution $P_{\mathcal{U}}$: its noise is marginalized out and consumed by the distributional divergences (JSD, KL, MMD), instantiating the interventional interpretation of \cae{} (Section~\ref{ssec:estimation}). The \emph{counterfactual} regime, with a nontrivial coupling $\tau_u$ that identifies corresponding exogenous realizations across the micro- and macro-levels, is out of scope and left for future work.

\paragraph{Design choices and inductive bias.} The benchmark necessarily involves design choices that can be debated: systems, valid abstractions, invalid contrasts, and sampled interventions. We mitigate this by trying to span diverse domains and abstraction types, but future versions should include community-contributed systems and possibly blind evaluations.


%% file: sections/07_acknowledgements.tex
\section*{Acknowledgements}

This work was partially conducted within French research unit UMR 5217 and was partially supported by CNRS (grant ANR-22-CPJ2-0036-01) and by MIAI@Grenoble-Alpes (grant ANR-19-P3IA-0003). It was granted access to the HPC resources of IDRIS under the allocation 2025-AD011014834 made by GENCI.\\

%% file: appendix/full_results.tex
\section{Additional results}
\label{app:results}

This section contains multiple additional figures for results described in the main body of the work, as well as a new experiment on convergence rates.

\subsection{Validity on Correct Abstractions}
\label{app:valid_scores}

Figure~\ref{fig:valid_scores} shows the scores produced by each baseline metric when applied to the correct abstractions in our benchmark.

\begin{figure}[htbp]
    \centering
    \includegraphics[width=\linewidth]{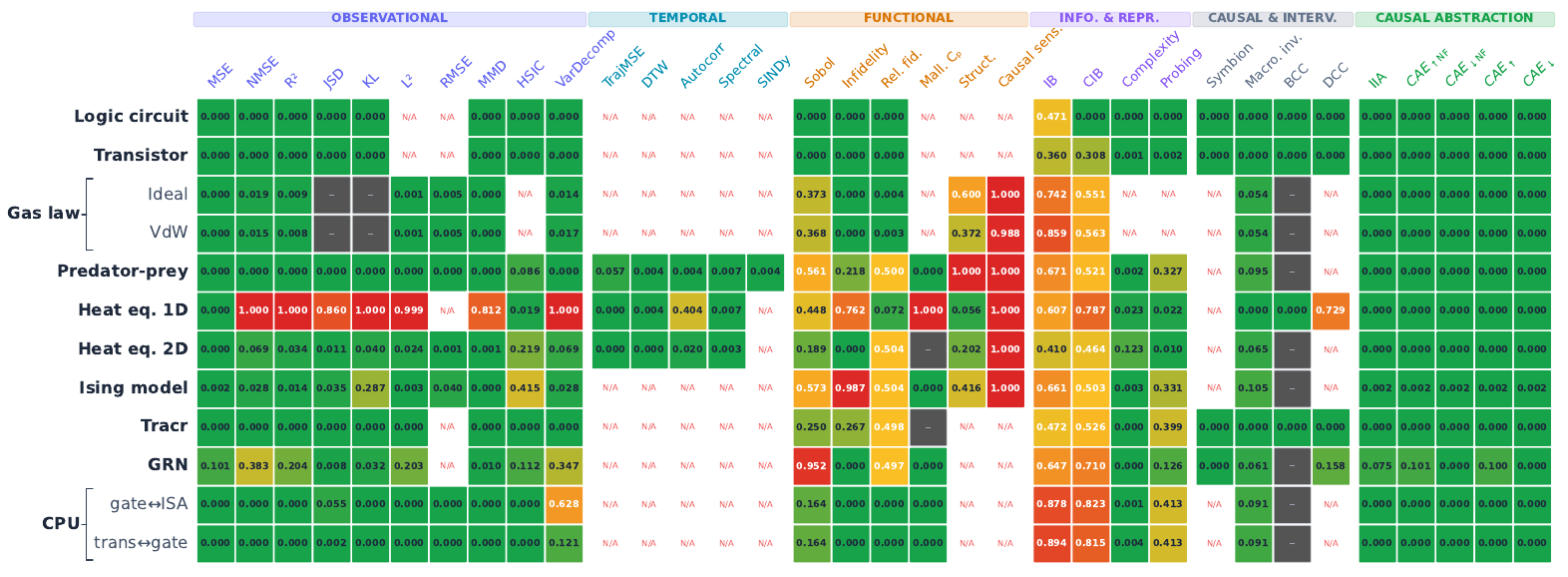}
    \caption{Error scores obtained by baseline metrics when applied to valid abstractions across benchmark systems. Scores are normalized to the $[0, 1]$ range, where $0$ denotes no error and $1$ maximal error. Gray cells denote computation failure or NaN results for the metric.}
    \label{fig:valid_scores}
\end{figure}

Note that most benchmark systems use an identity-style coarse-graining in which every micro-variable is either mapped or declared internal, leaving the unmapped set $\Phi=\emptyset$. For these systems, the faithfulness test has no $\Phi$-variable to perturb and is therefore vacuous, so the faithful (\caeup{}, \caedown{}) and non-faithful (\caeupnf{}, \caedownnf{}) variants coincide up to sampling noise; their separate entries should not be read as independent evidence. Only the logic circuit and the GRN have a non-empty $\Phi$ and thus a substantive faithfulness test.

\subsection{Discrimination Power on Invalid Abstractions}
\label{app:invalid_scores}

We report in Figure~\ref{fig:invalid_scores_full} the results of the discrimination measurement in \ref{ssec:discrimination}, applied to all systems (including controlled experiments) and metrics.

\begin{figure}[htbp]
    \centering
    \includegraphics[width=\linewidth]{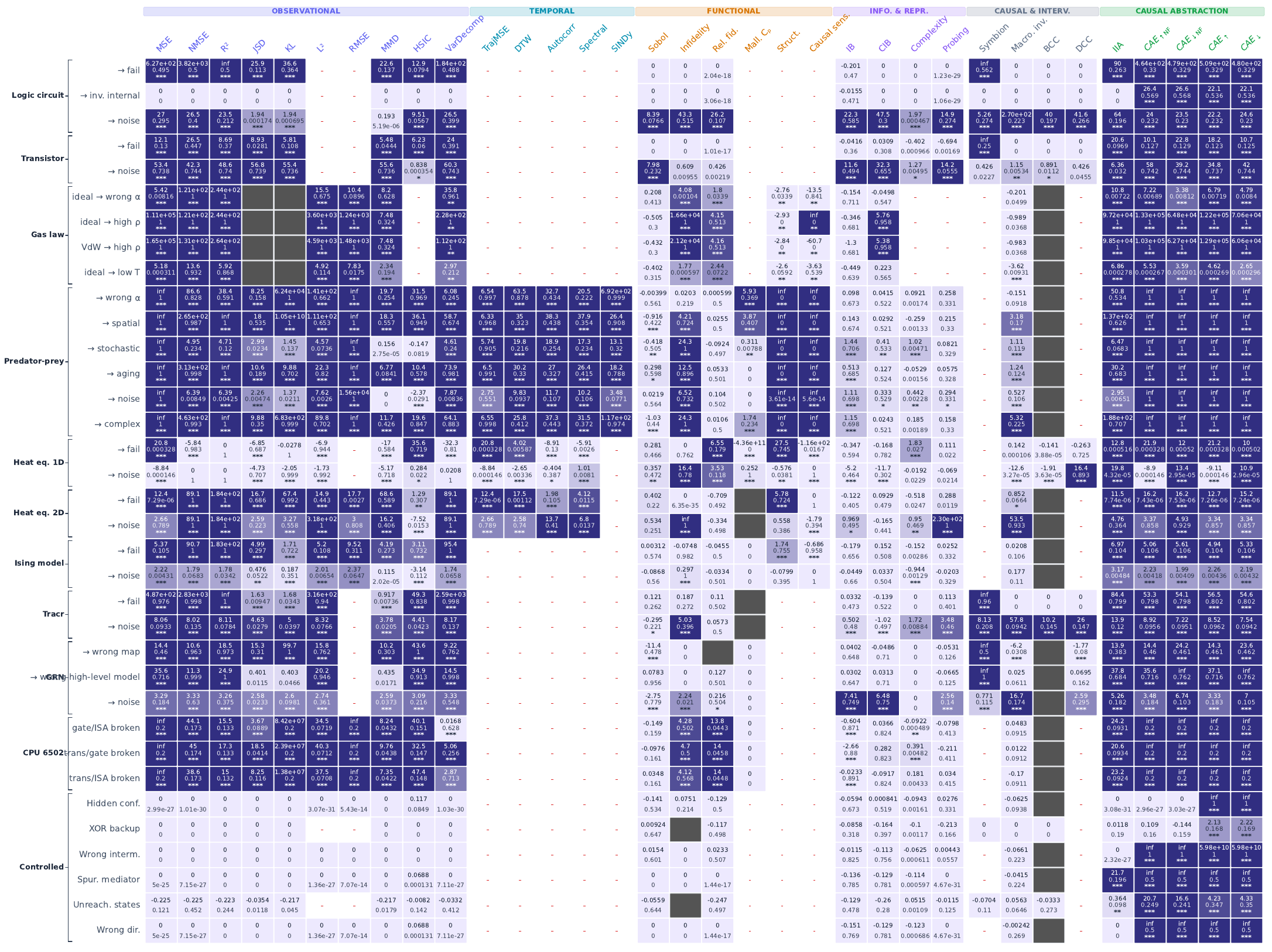}
    \caption{Discriminatory power of baseline metrics when applied to invalid abstractions compared to valid ones. For each cell, we report three values (from top to bottom): Cohen's $d$, computed on the scores obtained by the metric on the invalid vs. valid abstractions; the score obtained on the invalid condition; the significance level, computed on a Mann--Whitney $U$ test applied between the valid and invalid condition scores. Gray cells denote computation failure or NaN results for the metric. Stars represent p-values: p<0.05 (*), p<0.01 (**) and p<0.001 (***).}
    \label{fig:invalid_scores_full}
\end{figure}

\subsection{Power Curves}
\label{app:power_curves}

Figure~\ref{fig:power_curves_full} contains the full results of the power experiments in \ref{ssec:power}, computed across all invalid abstractions for all systems.

\begin{figure}[htbp]
    \centering
    \includegraphics[width=.9\linewidth]{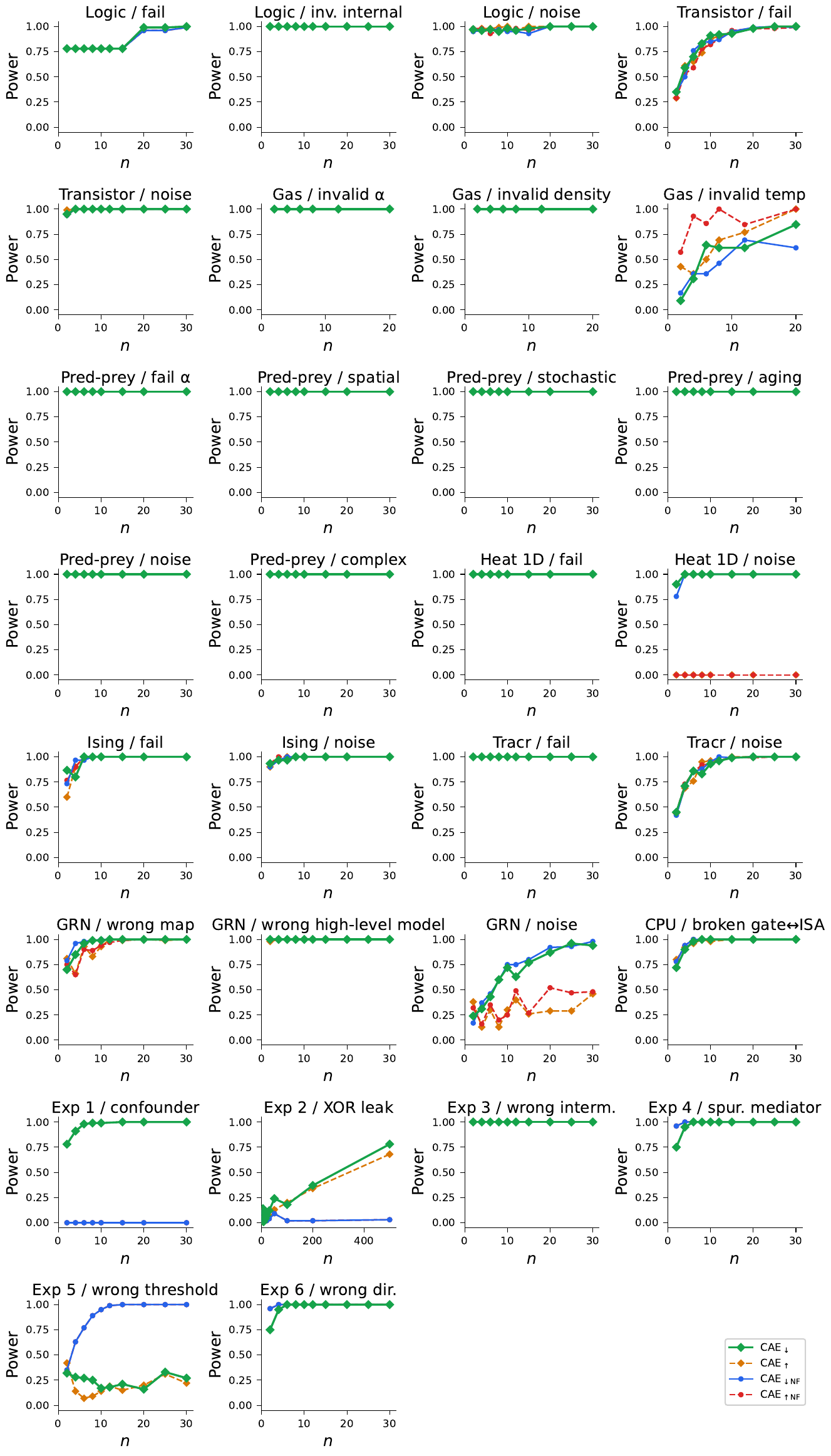}
    \caption{How quickly do metrics detect bad abstractions over good ones? For different numbers of sampled interventions $n$, we compute each metric 100 times on selected valid and invalid abstractions. We apply a Mann--Whitney $U$ test to the two sets of scores (valid and invalid) and report the fraction of runs for which the results differ significantly ($p<0.05$).}
    \label{fig:power_curves_full}
\end{figure}

Figure~\ref{fig:power_convergence_full} shows how the score estimates produced by each causal metric stabilize for different numbers of sampled interventions.

\begin{figure}[htbp]
    \centering
    \includegraphics[width=.9\linewidth]{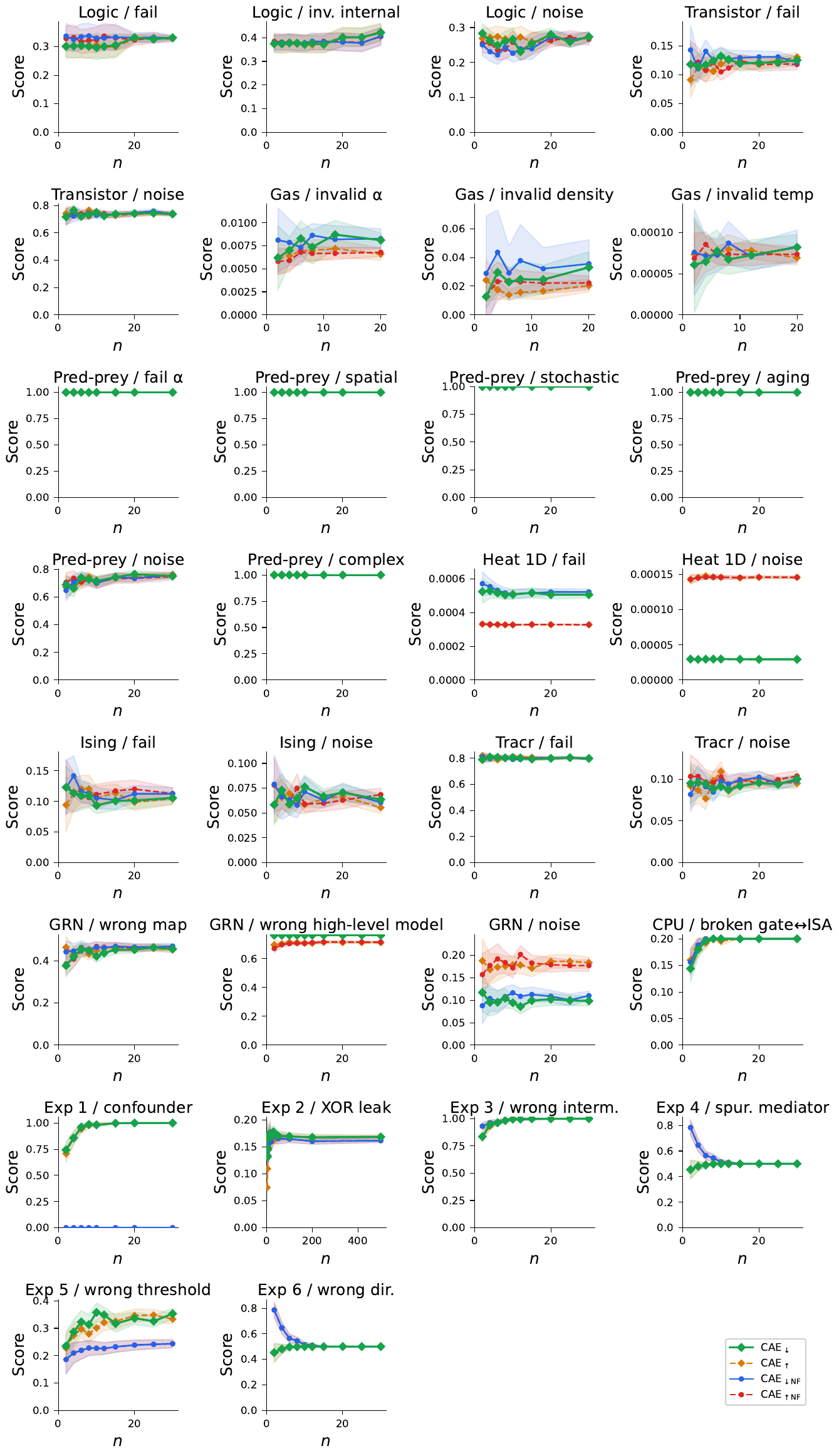}
    \caption{Convergence rate: For different numbers of sampled interventions $n$, we compute each metric 100 times on selected invalid abstractions.}
    \label{fig:power_convergence_full}
\end{figure}

Estimates stabilize quickly across most systems: 30 sampled interventions are generally sufficient to obtain a stable score, with the gas simulation generally requiring more samples due to the higher variance of the molecular dynamics estimator. No metric stands out as substantially slower to converge than the others. However, differences between sampling directions emerge. For the GRN with a reversed causal rule, top-down metrics (\caedown{} and \caedownnf{}), which intervene directly on the abstract states that expose the mis-specification, score higher than their bottom-up counterparts, which only draw micro-states reachable from the input distribution. For the GRN noise condition, bottom-up metrics score higher than top-down ones on valid abstractions: sampling generates interventions near the boundaries of the value-map subspaces, and noisy outputs of $\mathcal{M}$ are more likely to be misclassified after abstraction. Top-down sampling stays in the interior of each subspace, reducing the noise. The separation is notably strong in Experiment~1 (hidden confounder): \caenf{} variants score near zero, while faithfulness-augmented variants detect the invalid abstraction (score near one). This demonstrates that faithfulness testing is critical, independently of the sampling direction.

\subsection{Tracr power curves}
\label{app:tracr_scaling}

We provide here extended results for the Tracr sort-rank scaling experiment.

Figure~\ref{fig:tracr_power_curves_full} displays the full detection power curves across different sequence lengths. Figure~\ref{fig:tracr_scaling_heatmap} provides the minimal sample counts required to achieve a $95\%$ detection threshold.

\begin{figure}[htbp]
    \centering
    \includegraphics[width=\linewidth]{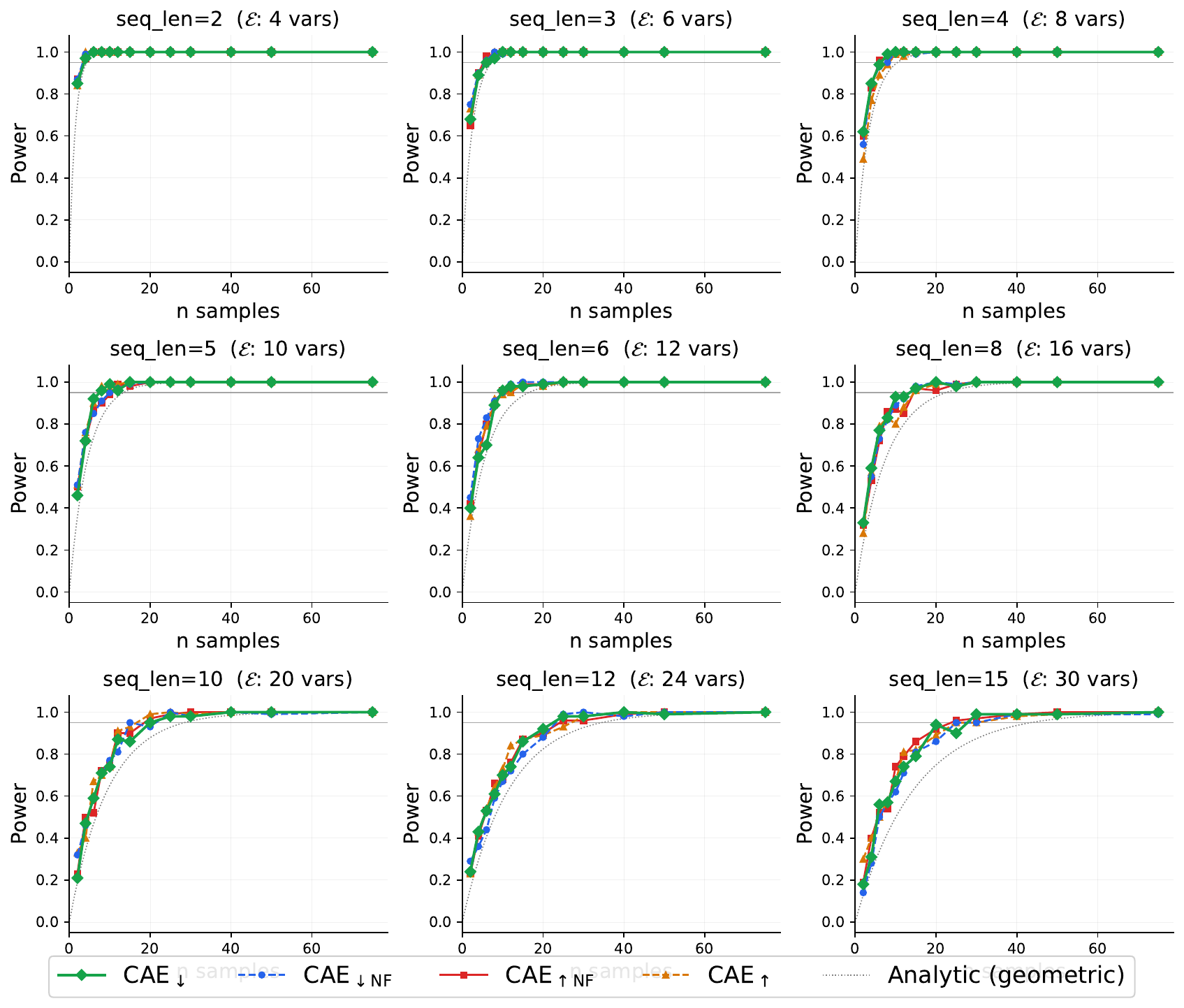}
    \caption{Detection power vs.\ the number of sampled interventions $n$ across different sequence lengths for the Tracr sort-rank programs. The curves show the empirical probability of successfully detecting the invalid condition (where $\text{rank}_0 = \ell-1$ when $\text{token}_0$ is the minimum).}
    \label{fig:tracr_power_curves_full}
\end{figure}

\begin{figure}[htbp]
    \centering
    \includegraphics[width=0.8\linewidth]{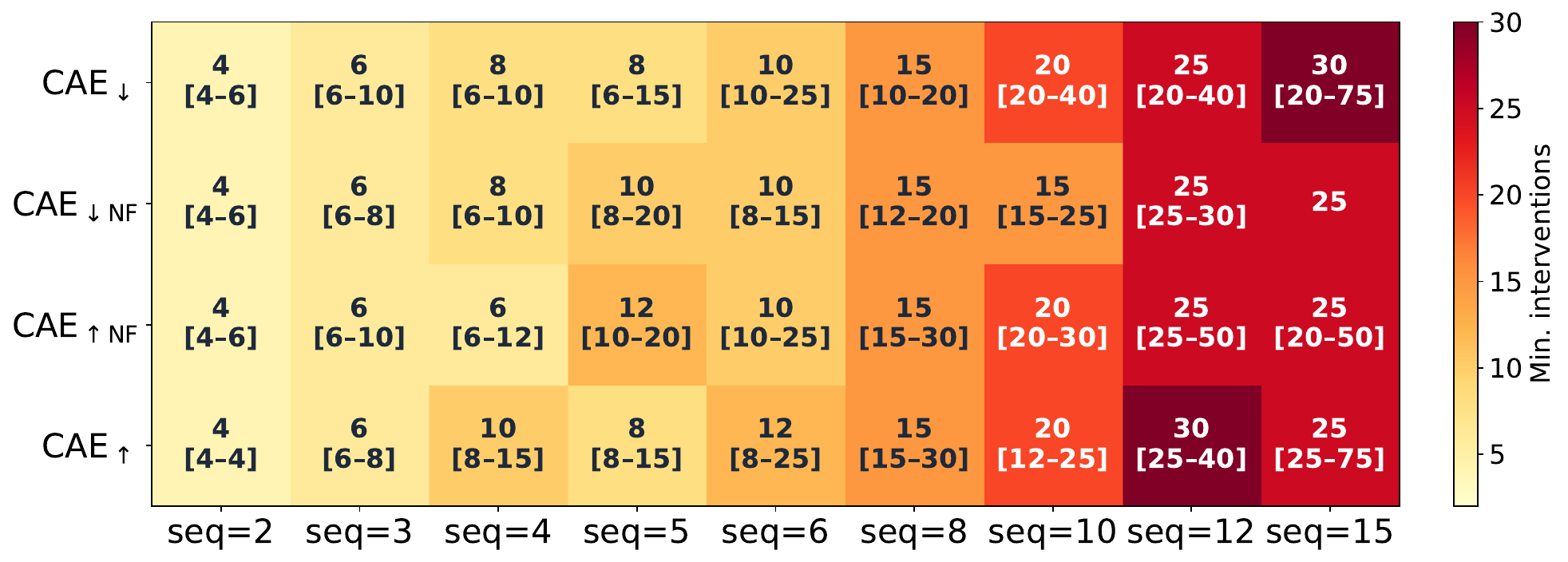}
    \caption{Minimum number of sampled interventions required to achieve 95\% detection power on the Tracr sort-rank experiment, evaluated across different metrics (rows) and sequence lengths (columns). Numbers between brackets represent Wilson confidence intervals.}
    \label{fig:tracr_scaling_heatmap}
\end{figure}

%% file: appendix/extended_related_work.tex
\section{Extended Background: Metrics of Explanation Validity}
\label{app:extended_background}

We survey the principal classes of metrics used to evaluate whether a computational explanation
$\mathcal{E}$ is a valid account of a target system $\mathcal{M}$.  Throughout, $u \sim P_{\mathcal{U}}$
denotes an exogenous input drawn from the natural input distribution, and validity is assessed
relative to the observable or internal behavior of $\mathcal{M}$.  We organize criteria by
the type of evidence they require: observational, functional, information-theoretic, or
causal.

\subsection{Observational Validity}
\label{sec:observational}

The most elementary criterion for explanation validity is \emph{observational equivalence}:
$\mathcal{E}$ is valid insofar as it reproduces the measurable outputs of $\mathcal{M}$ under the
natural input distribution.

\paragraph{Pointwise reconstruction accuracy.}
In deterministic or regression settings, observational validity is operationalized as low expected
point-to-point discrepancy between the outputs of $\mathcal{E}$ and $\mathcal{M}$.  Standard metrics
include the \emph{mean squared error}
$\mathrm{MSE} = \mathbb{E}_{u}[\|\mathcal{M}(u)-\mathcal{E}(u)\|^2]$,
the \emph{root mean squared error} $\mathrm{RMSE}$, the \emph{$\ell_2$ distance} $L_2$, and the
\emph{normalized} variant $\mathrm{NMSE}$, which removes
the dependence on output scale. The coefficient of determination $R^2$ provides an interpretable measure of the fraction of variance explained.  These criteria underlie
classical system identification~\citep{ljung1999} and general-purpose symbolic
regression~\citep{koza1994genetic,schmidt2009distilling,petersen2019deep,NEURIPS2022_dbca58f3},
where model selection is performed by minimizing reconstruction error on held-out data.

\paragraph{Complexity-regularized accuracy.}
Accuracy-only criteria permit arbitrarily complex explanations that overfit the observed data.
Standard practice augments the reconstruction objective with a complexity penalty $\Omega(\mathcal{E})$,
yielding criteria such as AIC~\citep{akaike2003new}, BIC~\citep{schwarz1978estimating},
and minimum description length~\citep{Mangan_2017,10.1145/3735634}.
\emph{Mallows' $C_p$} statistic~\citep{mallows1973,VALLEJO201448,https://doi.org/10.1155/2024/7741473,dst2009}
provides a related bias--variance decomposition, penalizing the number of free parameters:
$C_p = \mathrm{RSS}/\hat{\sigma}^2 - n + 2p$, where $\mathrm{RSS}$ is the residual sum of
squares, $p$ the number of parameters, and $\hat{\sigma}^2$ an estimate of noise variance.
In our setting, $C_p$ is applied to a linear proxy fitted between the predictions of $\mathcal{E}$ and the outputs of $\mathcal{M}$ so that it simultaneously penalizes residual mismatch and model complexity in a cross-model comparison: p is set to twice the number of variables in $\mathcal{E}$ (one slope and one intercept per variable), and $\sigma^2$ is estimated from the residuals of this proxy.
Physics-inspired approaches embed domain-specific inductive biases, such as dimensional homogeneity,
separability, known invariances, directly into the admissible hypothesis class, as in
AI~Feynman~\citep{udrescu2020aifeynmanphysicsinspiredmethod} and related
frameworks~\citep{Tenachi_2023,10.1145/3712255.3734309}. We also employ a \emph{SINDy-style evaluator}~\citep{brunton_discovering_2016}, which scores an explanation by how well its governing equations predict the empirical trajectory one epoch ahead. High residual indicates that the governing equations of $\mathcal{E}$ do not faithfully reproduce the system's epoch-level dynamics.

\paragraph{Distributional equivalence.}
When $\mathcal{M}$ is stochastic or the output of interest is a distributional quantity,
pointwise comparison is insufficient.  Observational validity is then defined by proximity
between the output marginals $P_{\mathcal{E}(u)}$ and $P_{\mathcal{M}(u)}$.  The
\emph{KL divergence} $D_{\mathrm{KL}}(P_\mathcal{M} \| P_\mathcal{E})$ and the symmetric
\emph{Jensen--Shannon divergence} $D_{\mathrm{JS}}$ measure discrepancy in a model-based sense.
The \emph{maximum mean discrepancy}
$\mathrm{MMD}^2(\mathcal{M},\mathcal{E}) = \|\mu_{P_\mathcal{M}} - \mu_{P_\mathcal{E}}\|^2_\mathcal{H}$,
where $\mu_P$ is the kernel mean embedding of $P$ in a reproducing kernel Hilbert space
$\mathcal{H}$~\citep{JMLR:v13:gretton12a}, provides a nonparametric alternative.
The \emph{Hilbert--Schmidt independence criterion} (HSIC)~\citep{gretton05} evaluates statistical dependence
between residuals $\mathcal{M}(u) - \mathcal{E}(u)$ and the input $u$; a low HSIC indicates that
errors are independent of the input, reflecting global rather than locally compensating agreement.
These criteria subsume the \emph{docking} paradigm in agent-based
modeling~\citep{Axtell1996,collins2024}, where models are accepted as equivalent when they
reproduce the same aggregate distributional patterns.

\paragraph{Temporal fidelity.}
For dynamical systems only, observational validity must be assessed over entire trajectories rather
than marginal snapshots.  \emph{Trajectory MSE} measures the average squared deviation between
state sequences $\{x_t^\mathcal{M}\}$ and $\{x_t^\mathcal{E}\}$ over a common time horizon.
\emph{Dynamic time warping} (DTW)~\citep{sakoe2003dynamic} generalizes this to temporally misaligned trajectories
by computing the minimum-cost elastic alignment:
$\mathrm{DTW}(\mathbf{x}^\mathcal{M}, \mathbf{x}^\mathcal{E}) = \min_{\pi} \sum_{(i,j)\in\pi} d(x_i^\mathcal{M}, x_j^\mathcal{E})$,
where $\pi$ ranges over admissible warping paths.  \emph{Temporal autocorrelation matching}
assesses whether $\mathcal{E}$ reproduces the autocorrelation structure
$\rho^\mathcal{M}(\tau) = \mathbb{E}[x_t x_{t+\tau}]$ of $\mathcal{M}$ across lags $\tau$;
failure indicates that $\mathcal{E}$ does not capture the system's memory structure.
\emph{Spectral analysis} compares the power spectral densities obtained via the Fourier transform of the respective autocorrelation functions~\citep{percival1993spectral}; agreement in dominant frequencies, harmonic structure, and spectral entropy indicates that the explanation captures the characteristic temporal modes of $\mathcal{M}$.

\paragraph{Limitation: equifinality.}
All observational criteria share a fundamental limitation: they are insensitive to the causal
mechanisms generating the data.  Distinct explanations may induce indistinguishable output
distributions, a phenomenon termed \emph{equifinality}~\citep{pmlr-v185-valogianni22a,collins2024}.
Observationally valid models can violate the causal structure of $\mathcal{M}$ while remaining
empirically accurate, leading to explanations that are unstable under distribution
shift~\citep{Ghorbani_Abid_Zou_2019,Kindermans2019,pmlr-v260-zhang25a} and
fragile to mechanistic perturbations~\citep{meloux2025mechanistic,meloux2025dead,katende2025causaloperatordiscoverypartial}.

\subsection{Functional Validity}
\label{sec:functional}

\emph{Functional validity} strengthens observational equivalence by requiring that $\mathcal{E}$
reproduce not only the outputs of $\mathcal{M}$ but also its \emph{input--output response profile}:
how outputs vary as inputs are systematically modified.  These criteria are most naturally applied
to static or quasi-static input--output systems, though several generalizations to dynamical
settings exist.

\paragraph{Variance decomposition and global sensitivity.}
\emph{Variance decomposition} decomposes the total output variance $\mathrm{Var}[\mathcal{M}(u)]$
into contributions attributable to individual inputs and their interactions via the ANOVA
expansion~\citep{efron1981jackknife,saltelli1999quantitative}.  \emph{Sobol sensitivity indices}~\citep{sobol2001global} formalize this as
$S_i = \mathrm{Var}[\mathbb{E}[\mathcal{M}(u)\mid u_i]] / \mathrm{Var}[\mathcal{M}(u)]$
(first-order) and analogously for higher-order terms (interactions between input coordinates). Under this validity criterion, $\mathcal{E}$
is deemed valid if its Sobol indices $\hat{S}_i$ agree with those of $\mathcal{M}$ up to an
acceptable tolerance, ensuring global functional coherence in how uncertainty in each input
propagates to the output.  Related global sensitivity methods including
FAST~\citep{cukier1973study,saltelli1999quantitative} and Morris
screening~\citep{morris1991factorial} provide computationally cheaper approximations to the same
underlying sensitivity structure.

\paragraph{Local attribution fidelity.}
A complementary notion evaluates agreement between $\mathcal{E}$ and $\mathcal{M}$ in a
neighborhood of a fixed input $u_0$.  The \emph{infidelity} metric~\cite{yeh2019fidelity}
is defined as
\[
  \mathrm{INFD}(\mathcal{E}, u_0)
  = \mathbb{E}_{\delta u}\!\left[
      \bigl(\delta u^\top \phi(\mathcal{E}, u_0)
      - [\mathcal{M}(u_0) - \mathcal{M}(u_0 - \delta u)]\bigr)^2
    \right],
\]
where $\phi(\mathcal{E}, u_0)$ is the attribution vector of $\mathcal{E}$ at $u_0$ and $\delta u$
is drawn from a perturbation distribution.  Low infidelity indicates that the explanation's
attributions accurately predict the actual output changes of $\mathcal{M}$ under local
perturbations. We adapt infidelity to compare output sensitivities between $\mathcal{E}$ and $\mathcal{M}$ directly, without requiring attribution vectors.  LIME~\cite{Ribeiro_lime} constructs a locally faithful surrogate by minimizing a
locally weighted reconstruction error around $u_0$, and SHAP~\cite{NIPS2017_8a20a862} defines
attributions via Shapley values, which satisfy a consistent axiomatic decomposition of the output
into per-feature contributions; both can be interpreted as proxies for local attribution fidelity.

\paragraph{Structural deviation and causal sensitivity index.}
While standard sensitivity indices quantify statistical influence under passive sampling, a complementary family of metrics evaluates the sensitivity of explanatory validity to the structural parameters of $\mathcal{E}$ itself~\citep{katende2025causaloperatordiscoverypartial}. The \emph{structural deviation} metric perturbs each parameter of $\mathcal{E}$ by a small amount and measures the resulting change in alignment score, identifying which parameters are most important for the explanation's validity.
The \emph{causal sensitivity index} applies a harder test: each parameter is zeroed out in turn, and the drop in alignment score is recorded, providing a global decomposition of validity across the components of $\mathcal{E}$.
Together, these metrics support model diagnosis by distinguishing robust explanatory components from those whose perturbation or removal would invalidate the abstraction.

\paragraph{Counterfactual and relational fidelity.}
\emph{Counterfactual validity}~\citep{DBLP:journals/corr/abs-1711-00399} requires that $\mathcal{E}$
reproduce the minimal perturbations $\delta u$ that change the class label or output regime of
$\mathcal{M}$: the decision boundaries of $\mathcal{E}$ must coincide with those of $\mathcal{M}$ in
the relevant neighborhood.
\emph{Relational fidelity}~\citep{collins2024} measures the Pearson correlation between $\Delta_\mathcal{E}(u, u') = \mathcal{E}(u') - \mathcal{E}(u)$ and $\Delta_\mathcal{M}(u, u') = \tau_{\mathcal{Y}}(\mathcal{M}(u')) - \tau_{\mathcal{Y}}(\mathcal{M}(u))$ over sampled pairs $(u, u')$, computed independently per output dimension and averaged. A score of 0 (after normalization) indicates perfect agreement in relational structure.

\paragraph{Limitation: structural underdetermination.}
Functional criteria remain agnostic to internal mechanisms.
Distinct models can exhibit identical global sensitivity profiles, local attributions, and
counterfactual boundaries while relying on entirely different underlying structures. Functional
validity therefore does not uniquely identify the causal organization of $\mathcal{M}$, and
multiple mechanistically incompatible explanations may be simultaneously admissible.

\subsection{Information-Theoretic and Representational Validity}
\label{sec:infotheoretic}

Another class of criteria requires that $\mathcal{E}$ reproduce not only the
input--output behavior of $\mathcal{M}$ but also the structure of its internal
representations: how information is encoded, compressed, and transmitted across the system.

\paragraph{Representational alignment and probing accuracy.}
A first approach tests whether the abstract variables posited by $\mathcal{E}$ are linearly
decodable from the internal states of $\mathcal{M}$. In the case of neural networks, given a representation
$h(u) \in \mathbb{R}^d$ extracted from a nominated layer of $\mathcal{M}$, \emph{probing
accuracy} measures the performance of a linear classifier or regressor trained to predict an
explanatory variable from $h(u)$~\citep{alain2016understanding}. For real neural networks, neuroscience commonly refers to these practices as \textit{multivariate pattern analysis}~\citep{haxby2001distributed,norman2006beyond} or \textit{brain signatures}~\citep{wager2013fmri,kragel2018representation}. High probing
accuracy indicates that the feature is explicitly represented in $\mathcal{M}$'s internal geometry~\citep{ravichander-etal-2021-probing,belinkov-2022-probing}.
For systems with continuous identity value maps, where all samples share a single abstract label, probing accuracy is trivially maximal and reflects the smoothness of the micro-model's representation rather than the discriminality of abstract categories.
Supervised concept-level probing is instantiated by Concept Activation
Vectors~\citep{kim2018interpretabilityfeatureattributionquantitative}; unsupervised variants
based on sparse autoencoders (SAEs)~\citep{cunningham2023sparseautoencodershighlyinterpretable,bricken2023monosemanticity}
recover a dictionary of monosemantic features without requiring labeled concepts.
Representational similarity analysis (RSA)~\citep{kriegeskorte2008representational} generalizes
this by comparing representational dissimilarity matrices between $\mathcal{E}$ and $\mathcal{M}$~\citep{schrimpf2021neural},
without committing to a specific linear structure~\citep{aw2024instructiontuningalignsllmshuman,ryskina2025language}.

\paragraph{Information bottleneck and causal information bottleneck Lagrangians.}
The \emph{information bottleneck} (IB) framework characterizes valid representations as those
achieving an optimal trade-off between compression of the input and preservation of information relevant to the output~\citep{tishby2000informationbottleneckmethod}.  
However, standard IB relies on
observational mutual information, rendering it sensitive to spurious correlations in the input
distribution.  The \emph{causal information bottleneck} (CIB)~\citep{simoes2025causalinformationbottleneckoptimal}
strengthens this by replacing observational relevance with causal relevance measured under
interventions over inputs, so that $\mathcal{E}$ is assessed by its capacity to preserve causally invariant predictive information.
Operationally, the IB Lagrangian is evaluated using the observational mutual information $I(\mathcal{E}(u); \mathcal{M}(u))$, whereas the CIB Lagrangian replaces this with the interventional relevance $H(Y) - H_c(Y \mid \mathrm{do}(T))$; this distinction is relevant when the input distribution contains spurious correlations that increase the apparent observational relevance of $\mathcal{E}$.

\paragraph{Complexity shift.}
\emph{Algorithmic information dynamics} (AID)~\citep{ZENIL20191160,Zenil:2020,Zenil_Kiani_Tegnér_2023}
evaluates explanations in terms of changes in algorithmic complexity induced by interventions.
The \emph{complexity shift} metric is defined as the discrepancy between the change in Kolmogorov
complexity $K$ of the output under perturbation of $\mathcal{M}$ versus the corresponding change in
$\mathcal{E}$:
\[
  \Delta K_\mathcal{M}(u, \delta u) = K(\mathcal{M}(u+\delta u)) - K(\mathcal{M}(u)), \qquad
  \mathrm{CS}(\mathcal{E}) = \mathbb{E}\!\left[\bigl|\Delta K_\mathcal{M} - \Delta K_\mathcal{E}\bigr|\right].
\]
Low complexity shift indicates that $\mathcal{E}$ produces mechanistically comparable
transformations under perturbation.  In practice, Kolmogorov complexity is approximated via
compression algorithms or block decomposition methods~\citep{ZENIL20191160}, introducing
estimation variance.

\paragraph{Limitation: mechanistic underdetermination.}
Information-theoretic criteria capture \emph{what information} is present and how it is
organized across variables, but not \emph{how} that information is causally transformed by the
system's dynamics.  Distinct mechanisms can yield identical probing accuracies, identical IB
optima, and identical complexity-shift profiles~\citep{meloux2025dead}.  
Moreover, most of these criteria, except in
their explicitly causal variants, are grounded in observational regimes~\citep{exp_unders}.

\subsection{Causal and Interventional Validity}
\label{sec:causal}

The strictest class of criteria requires that $\mathcal{E}$ correctly reproduce the behavior of
$\mathcal{M}$ under \emph{active interventions}. Informally, let $\mathcal{I}$ be a class of interventions
on $\mathcal{M}$, and let $\mathcal{M}^{(i)}$ and $\mathcal{E}^{(i)}$ denote the respective
post-intervention systems for $i \in \mathcal{I}$.  Then $\mathcal{E}$ is \emph{interventionally
valid} relative to $\mathcal{I}$ if
$ P_{\mathcal{M}^{(i)}}(Y) \;\approx\; P_{\mathcal{E}^{(i)}}(\hat{Y}), \: \forall\, i \in \mathcal{I}, $
where $Y$ and $\hat{Y}$ are corresponding output variables under a fixed variable correspondence.
Unlike observational criteria, this requirement is closed under the causal operator: agreement
must hold across a family of intervention regimes.

\paragraph{Behavioral causal consistency and macroscopic invariance.}
A prerequisite for interventional validity is that the abstraction map $\tau$ must be well-defined with respect to the dynamics: distinct micro-states sharing the same abstract label must produce the same abstract output under intervention, regardless of the micro-realization.
In the Neuro-Cognitive Multilevel Causal Modeling framework~\citep{nc-mcm}, this is formalized as \emph{behavioral causal consistency} (BCC): for every pair $x_1, x_2 \in \tau^{-1}(c), P(B[t] \mid \mathrm{do}(X[t] = x_1)) = P(B[t] \mid \mathrm{do}(X[t] = x_2))$. The same requirement appears independently in statistical physics under the name \emph{macroscopic invariance}: a coarse-graining is admissible only when all micro-configurations within a macro-state are observationally equivalent at the macro-level~\citep{RevModPhys.70.653}.
BCC and macroscopic invariance differ in scope: macroscopic invariance tests each variable in isolation, intervening on a single variable while leaving others unspecified, whereas BCC tests each variable within a complete system state drawn from the joint distribution. Operationally, these are assessed by sampling pairs of micro-states with the same abstract label and verifying that their abstracted outputs of $\mathcal{M}$ agree under intervention.

\paragraph{Dynamic causal consistency.}
The \emph{dynamic causal consistency} (DCC) criterion, introduced in the Neural Causal Model
framework for dynamical systems~\citep{nc-mcm}, requires that interventions on $\mathcal{E}$ correctly predict the
next-step internal states of $\mathcal{M}$.  Formally, DCC is satisfied if
$\mathcal{E}^{(i)}(u_t) \approx h_\mathcal{M}(u_{t+1})$ for all interventions $i \in \mathcal{I}$
and time steps $t$, where $h_\mathcal{M}(\cdot)$ denotes the corresponding internal state of
$\mathcal{M}$.  DCC thus enforces step-wise causal coherence between the explanation's dynamics
and those of the target system.  It is a dynamical extension of behavioral causal consistency (BCC): where BCC requires that same-label micro-states produce identical behavioral outputs under intervention, DCC additionally requires that micro-states sharing the same abstract label at time $t$ must also do so at $t+1$ under the system dynamics. We evaluate this using exact label matching on the abstracted next state.

\paragraph{Constructive and formal verification.}
The strongest guarantees arise in settings where $\mathcal{E}$ can be analytically derived from
or formally verified against $\mathcal{M}$.  In physics, renormalization group methods and
effective field theory~\citep{RevModPhys.70.653,georgi_effective_1993} construct coarse-grained
explanations that provably preserve long-distance or low-energy observables. In computer science, formal
verification~\citep{10.1145/1592434.1592436} establishes that an implementation $\mathcal{M}$
satisfies a high-level specification $\mathcal{E}$ for \emph{all} inputs, yielding universal behavioral
equivalence guarantees.  \emph{Symbion}~\citep{gritti2019symbion} provides a useful source of inspiration through its use of agreement between symbolic and concrete execution as a consistency check. Adapting this idea to the explanation setting, we treat the explanatory model $\mathcal{E}$ as an abstract representation of the target model $\mathcal{M}$ and measure validity by the extent to which their behaviors agree. In our evaluation, we implement a discrete analog suited to finite abstract domains: we exhaustively enumerate all combinations of abstract input labels, execute both models, and measure the fraction of input combinations for which they disagree on any output.
In causal agent-based modeling~\citep{pmlr-v185-valogianni22a}, explanations are constructed to
satisfy domain-theoretic structural constraints.

\paragraph{Mechanistic interpretability.}
Developed primarily in AI, mechanistic interpretability (MI) aims to reverse-engineer networks into human-interpretable algorithms \citep{olah2020zoom}. MI assumes that for a given behavior of interest, a sparse subset of the network executes the relevant algorithm. This subset, often referred to as a \textit{circuit}, is a key research target of modern AI interpretability methods \cite{NEURIPS2020_92650b2e, meng2022locating, monea2024glitchmatrixlocatingdetecting, kramár2024atpefficientscalablemethod, NEURIPS2023_34e1dbe9, geva-etal-2023-dissecting, syed2023attributionpatchingoutperformsautomated}. It is loosely inspired by neuroscience, which also seeks to uncover neural circuits underlying observed behaviors \cite{yuste2008circuit}. A circuit is valid if it reproduces the network output under interventions on the computation, whether they are part or not of the circuit~\citep{hanna2024have}. \emph{Interchange intervention accuracy} (IIA)~\citep{pmlr-v162-geiger22a} generalizes this perspective: it requires that pairwise interchange interventions on the network's internal representations and the high-level model produce consistent outputs. Unlike IC \citep{zennaro_quantifying_2023} and its variants, IIA evaluates interventional consistency without requiring a fully specified abstraction map over variables.

\paragraph{Computational mechanics and $\epsilon$-machines.}
A more general perspective is provided by computational mechanics, in which valid explanations
correspond to \emph{causally closed} macro-level representations.  The $\epsilon$-machine
framework~\citep{PhysRevLett.63.105,shalizi_computational_2001} defines the canonical valid
explanation as the minimal sufficient statistic for predicting $\mathcal{M}$: the set of
\emph{causal states} $\mathcal{S}$ such that future behavior is conditionally independent of the
past given $\mathcal{S}$.  An explanation is valid under this criterion if its state space is
isomorphic to the $\epsilon$-machine of $\mathcal{M}$, ensuring that no micro-level information
would improve predictive accuracy.  Related perspectives interpret valid explanations as
\emph{programs implemented} by the underlying system, whose macro-level dynamics are
self-contained and interventionally sufficient~\citep{rosas2024softwarenaturalworldcomputational}.

\paragraph{Limitations}
While causal and interventional criteria provide the strongest notion of validity, existing
formulations face practical limitations in complex, stochastic settings.  Constructive approaches
require strong domain-specific assumptions. Mechanistic interpretability operates at the wrong level of abstraction, viewing low-level components as variables of the high-level explanation and, thus, suffers from causal overdetermination~\citep{mcgrath2023hydraeffectemergentselfrepair,meloux2025dead}. Finally, computational mechanics and $\epsilon$-machines are subsumed by the more general framework of causal abstraction.

%% file: appendix/controlled_experiments.tex
\section{Controlled Experiments}
\label{app:controlled_exp}

We describe below the six controlled experiments from Figure~\ref{fig:controlled_graphs}. These highlight subtle ways in which abstraction can fail: aside from the strongest causal abstraction metrics, baseline metrics generally fail to detect the errors in these claimed abstractions.

Each panel in this appendix shows the full structural model, including the exogenous noise variables $u_{\bullet}$. Every endogenous variable carries its own exogenous noise. The noise abstraction map $\tau_u$ is the identity: each micro-noise $u_v$ maps to the macro-noise of the corresponding variable. For $\Phi$ variables (unmapped by $\tau$), the noise has no macro counterpart, so $\tau_u$ is undefined and the corresponding noise is held at zero under the valid abstraction. As in Figure~\ref{fig:controlled_graphs}, $\mathcal{E}$ is drawn on top (uppercase) and $\mathcal{M}$ on the bottom (lowercase), gray dashed arrows are $\tau$, dashed gray circles are $\Phi$ variables, and red marks what changes in the invalid condition.

\begin{figure}[h]
\centering

\begin{minipage}{0.32\textwidth}
\centering
\begin{tikzpicture}
  \node[hn] (X1) at (0,0)    {$X_1$};
  \node[hn] (Y)  at (15mm,0) {$Y$};
  \node[hn] (X2) at (30mm,0) {$X_2$};
  \draw[->,thick] (X1)--(Y);
  \draw[->,thick] (X2)--(Y);
  \node[lb,left=1pt of X1] {$\mathcal{E}$:};
  \node[nz] (uX1) at (-1mm,7mm)  {$u_{X_1}$}; \draw[ne] (uX1)--(X1);
  \node[nz] (uY)  at (16mm,7mm)  {$u_{Y}$};   \draw[ne] (uY)--(Y);
  \node[nz] (uX2) at (31mm,7mm)  {$u_{X_2}$}; \draw[ne] (uX2)--(X2);
  \node[ln] (x1) at (0,-18mm)    {$x_1$};
  \node[ln] (y)  at (15mm,-18mm) {$y$};
  \node[ln] (x2) at (30mm,-18mm) {$x_2$};
  \draw[->] (x1)--(y);
  \draw[->] (x2)--(y);
  \node[lb,left=1pt of x1] {$\mathcal{M}$:};
  \node[nz] (ux1) at (-1mm,-25mm) {$u_{x_1}$}; \draw[ne] (ux1)--(x1);
  \node[nz] (uy)  at (10mm,-25mm) {$u_{y}$};   \draw[ne] (uy)--(y);
  \node[nz] (ux2) at (31mm,-25mm) {$u_{x_2}$}; \draw[ne] (ux2)--(x2);
  \node[pn] (r)  at (24mm,-9mm) {$r$};
  \node[nzi] (ur) at (33mm,-9mm) {$u_{r}$}; \draw[nei] (ur)--(r);
  \draw[->,iv] (r)--(y) node[lb,text=red!75,pos=0.5,above,sloped]{$1{+}\gamma r$};
  \draw[ae] (x1)--(X1);
  \draw[ae] (y)--(Y);
  \draw[ae] (x2)--(X2);
\end{tikzpicture}
\end{minipage}%
\hfill
\begin{minipage}{0.32\textwidth}
\centering
\begin{tikzpicture}
  \node[hn] (X) at (0,0)    {$X$};
  \node[hn] (M) at (15mm,0) {$M$};
  \node[hn] (Y) at (30mm,0) {$Y$};
  \draw[->,thick] (X)--(M);
  \draw[->,thick] (M)--(Y);
  \node[lb,left=1pt of X] {$\mathcal{E}$:};
  \node[nz] (uX) at (-1mm,7mm) {$u_{X}$}; \draw[ne] (uX)--(X);
  \node[nz] (uM) at (15mm,7mm) {$u_{M}$}; \draw[ne] (uM)--(M);
  \node[nz] (uY) at (31mm,7mm) {$u_{Y}$}; \draw[ne] (uY)--(Y);
  \node[ln] (x)  at (0,-18mm)    {$x$};
  \node[ln] (a)  at (15mm,-18mm) {$m$};
  \node[ln] (y)  at (30mm,-18mm) {$y$};
  \node[pn] (b1) at (9mm,-9mm)   {$a$};
  \node[pn] (b2) at (9mm,-27mm)  {$b$};
  \draw[->] (x)--(a);
  \draw[->] (a)--(y);
  \draw[->] (x)--(b1);
  \draw[->] (x)--(b2);
  \node[lb,left=1pt of x] {$\mathcal{M}$:};
  \node[nz] (ux) at (-1mm,-25mm) {$u_{x}$}; \draw[ne] (ux)--(x);
  \node[nz] (ua) at (15mm,-25mm) {$u_{m}$}; \draw[ne] (ua)--(a);
  \node[nz] (uy) at (31mm,-25mm) {$u_{y}$}; \draw[ne] (uy)--(y);
  \node[nzi] (ub1) at (1mm,-9mm)  {$u_a$}; \draw[nei] (ub1)--(b1);
  \node[nzi] (ub2) at (1mm,-27mm) {$u_b$}; \draw[nei] (ub2)--(b2);
  \draw[->,iv] (b1) to[bend left=10]  (y);
  \draw[->,iv] (b2) to[bend right=10] (y);
  \node[lb,text=red!75] at (26mm,-11mm) {$m\!\lor\!(a\!\oplus\!b)$};
  \draw[ae] (x)--(X);
  \draw[ae] (a)--(M);
  \draw[ae] (y)--(Y);
\end{tikzpicture}
\end{minipage}%
\hfill
\begin{minipage}{0.32\textwidth}
\centering
\begin{tikzpicture}
  \node[hn] (X) at (0,0)    {$X$};
  \node[hn] (M) at (13mm,0) {$M$};
  \node[hn] (Z) at (26mm,0) {$Z$};
  \draw[->,thick] (X)--(M);
  \draw[->,thick] (M)--(Z);
  \node[lb,left=1pt of X] {$\mathcal{E}$:};
  \node[lb]            at (6mm, 5mm)  {$M{=}2X$};
  \node[lb,text=red!75] at (6mm,-5mm)  {$M{=}3X$};
  \node[lb]            at (20mm, 5mm) {$Z{=}M{+}3$};
  \node[lb,text=red!75] at (20mm,-5mm) {$Z{=}\tfrac{2}{3}M{+}3$};
  \node[nz] (uX) at (-1mm,8mm) {$u_{X}$}; \draw[ne] (uX)--(X);
  \node[nz] (uM) at (13mm,9mm) {$u_{M}$}; \draw[ne] (uM)--(M);
  \node[nz] (uZ) at (27mm,8mm) {$u_{Z}$}; \draw[ne] (uZ)--(Z);
  \node[ln] (x) at (0,-18mm)    {$x$};
  \node[ln] (m) at (13mm,-18mm) {$m$};
  \node[ln] (z) at (26mm,-18mm) {$z$};
  \draw[->] (x)--(m);
  \draw[->] (m)--(z);
  \node[lb,left=1pt of x] {$\mathcal{M}$:};
  \node[lb] at (6mm,-13mm)  {$m{=}2x$};
  \node[lb] at (20mm,-13mm) {$z{=}m{+}3$};
  \node[nz] (ux) at (-1mm,-25mm) {$u_{x}$}; \draw[ne] (ux)--(x);
  \node[nz] (um) at (13mm,-25mm) {$u_{m}$}; \draw[ne] (um)--(m);
  \node[nz] (uz) at (27mm,-25mm) {$u_{z}$}; \draw[ne] (uz)--(z);
  \draw[ae] (x)--(X);
  \draw[ae] (m)--(M);
  \draw[ae] (z)--(Z);
\end{tikzpicture}
\end{minipage}

\caption{Experiments 1 to 3.}
\label{fig:exp_1_3}
\end{figure}
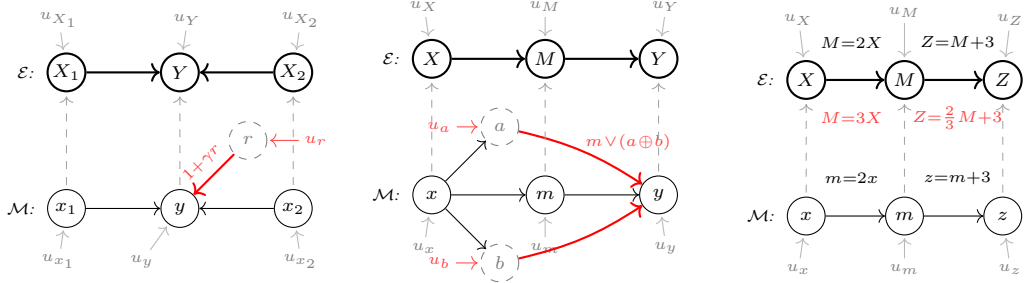

\paragraph{Experiment 1: Hidden confounder.}
The macro-model posits that initial prey and predator populations determine the final populations (Lotka-Volterra dynamics). $\mathcal{M}$ introduces an unmapped resource variable $r$ (a $\Phi$ variable) that multiplicatively scales the final prey population output by a factor of $1 + \gamma r$, where $\gamma$ is the coupling strength and $r$ is the resource level. This creates a hidden confounder that the macro-model cannot account for. The confounder enters through the exogenous noise $u_r$ of the unmapped variable, which has no image under $\tau_u$. A valid metric should detect that the abstraction fails under perturbations to this resource variable. In the valid condition $\gamma=0$, so the factor is $1$ and the resource is inert.

\paragraph{Experiment 2: XOR backup path.}
The macro-model posits $X \to M \to Y$, with two $\Phi$ variables $a, b$ (declared causally inert). In the valid $\mathcal{M}$, the XOR path is structurally disconnected ($y = m = x$), so perturbing $a$ and $b$ has no effect regardless of their values. In the invalid $\mathcal{M}$, $Y$ is computed as $m \lor (a \oplus b)$, where $m$ is the micro-variable corresponding to $M$: a hidden XOR backup path to $Y$ fires only when the $\phi$ variables are independently perturbed through their noise $u_a, u_b$. This experiment specifically tests faithfulness: structural metrics that do not probe $\phi$ variables cannot detect the violation.
 
\paragraph{Experiment 3: Wrong intermediate variable.}
$\mathcal{M}$ implements the chain $x \to m = 2x \to z = m + 3$. The valid $\mathcal{E}$ matches this exactly. The invalid $\mathcal{E}$ posits $m = 3x$ with a compensating downstream equation $(z = \tfrac{2}{3}m + 3)$, so that $z = 2x + 3$ in both cases. The intermediate $m$ is wrong, but the output is identical under observation.

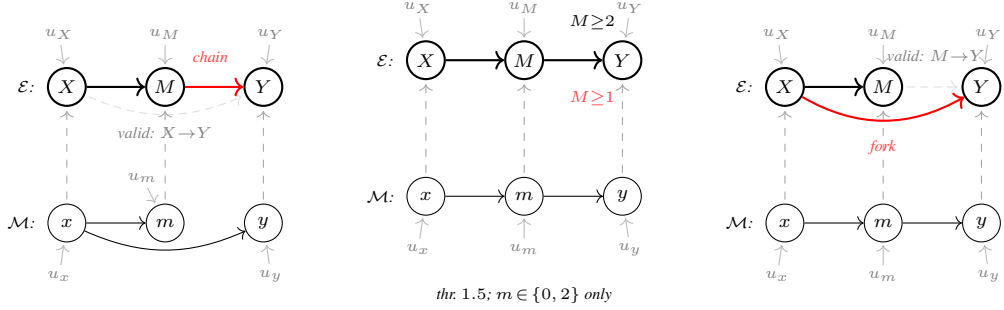
\begin{figure}[h]
\centering

\begin{minipage}{0.32\textwidth}
\centering
\begin{tikzpicture}
  \node[hn] (X) at (0,0)    {$X$};
  \node[hn] (Y) at (13mm,0) {$M$};
  \node[hn] (Z) at (26mm,0) {$Y$};
  \draw[->,thick] (X)--(Y);
  \draw[rm] (X) to[bend right=24] (Z);
  \node[lb,gray] at (13mm,-6mm) {valid: $X{\to}Y$};
  \draw[->,iv] (Y)--(Z);
  \node[lb,text=red!75] at (19mm,4mm) {chain};
  \node[lb,left=1pt of X] {$\mathcal{E}$:};
  \node[nz] (uX) at (-1mm,7mm) {$u_{X}$}; \draw[ne] (uX)--(X);
  \node[nz] (uY) at (13mm,7mm) {$u_{M}$}; \draw[ne] (uY)--(Y);
  \node[nz] (uZ) at (27mm,7mm) {$u_{Y}$}; \draw[ne] (uZ)--(Z);
  \node[ln] (x) at (0,-18mm)    {$x$};
  \node[ln] (y) at (13mm,-18mm) {$m$};
  \node[ln] (z) at (26mm,-18mm) {$y$};
  \draw[->] (x)--(y);
  \draw[->] (x) to[bend right=24] (z);
  \node[lb,left=1pt of x] {$\mathcal{M}$:};
  \node[nz] (ux) at (-1mm,-25mm) {$u_{x}$}; \draw[ne] (ux)--(x);
  \node[nz] (uy) at (10mm,-12mm) {$u_{m}$}; \draw[ne] (uy)--(y);
  \node[nz] (uz) at (27mm,-25mm) {$u_{y}$}; \draw[ne] (uz)--(z);
  \draw[ae] (x)--(X);
  \draw[ae] (y)--(Y);
  \draw[ae] (z)--(Z);
\end{tikzpicture}
\end{minipage}%
\hfill
\begin{minipage}{0.32\textwidth}
\centering
\begin{tikzpicture}
  \node[hn] (X) at (0,0)    {$X$};
  \node[hn] (M) at (13mm,0) {$M$};
  \node[hn] (Z) at (26mm,0) {$Y$};
  \draw[->,thick] (X)--(M);
  \draw[->,thick] (M)--(Z);
  \node[lb,left=1pt of X] {$\mathcal{E}$:};
  \node[lb]            at (22mm, 5mm) {$M{\ge}2$};
  \node[lb,text=red!75] at (22mm,-5mm) {$M{\ge}1$};
  \node[nz] (uX) at (-1mm,7mm) {$u_{X}$}; \draw[ne] (uX)--(X);
  \node[nz] (uM) at (13mm,7mm) {$u_{M}$}; \draw[ne] (uM)--(M);
  \node[nz] (uZ) at (27mm,7mm) {$u_{Y}$}; \draw[ne] (uZ)--(Z);
  \node[ln] (x) at (0,-18mm)    {$x$};
  \node[ln] (m) at (13mm,-18mm) {$m$};
  \node[ln] (z) at (26mm,-18mm) {$y$};
  \draw[->] (x)--(m);
  \draw[->] (m)--(z);
  \node[lb,left=1pt of x] {$\mathcal{M}$:};
  \node[nz] (ux) at (-1mm,-25mm) {$u_{x}$}; \draw[ne] (ux)--(x);
  \node[nz] (um) at (13mm,-25mm) {$u_{m}$}; \draw[ne] (um)--(m);
  \node[nz] (uz) at (27mm,-25mm) {$u_{y}$}; \draw[ne] (uz)--(z);
  \node[lb] at (13mm,-31mm) {thr.\ $1.5$;\ $m\!\in\!\{0,2\}$ only};
  \draw[ae] (x)--(X);
  \draw[ae] (m)--(M);
  \draw[ae] (z)--(Z);
\end{tikzpicture}
\end{minipage}%
\hfill
\begin{minipage}{0.32\textwidth}
\centering
\begin{tikzpicture}
  \node[hn] (X) at (0,0)    {$X$};
  \node[hn] (Y) at (13mm,0) {$M$};
  \node[hn] (Z) at (26mm,0) {$Y$};
  \draw[->,thick] (X)--(Y);
  \draw[rm] (Y)--(Z);
  \node[lb,gray] at (20mm,4.2mm) {valid: $M{\to}Y$};
  \draw[->,iv] (X) to[bend right=30] (Z);
  \node[lb,text=red!75] at (13mm,-8mm) {fork};
  \node[lb,left=1pt of X] {$\mathcal{E}$:};
  \node[nz] (uX) at (-1mm,7mm) {$u_{X}$}; \draw[ne] (uX)--(X);
  \node[nz] (uY) at (13mm,7mm) {$u_{M}$}; \draw[ne] (uY)--(Y);
  \node[nz] (uZ) at (27mm,7mm) {$u_{Y}$}; \draw[ne] (uZ)--(Z);
  \node[ln] (x) at (0,-18mm)    {$x$};
  \node[ln] (y) at (13mm,-18mm) {$m$};
  \node[ln] (z) at (26mm,-18mm) {$y$};
  \draw[->] (x)--(y);
  \draw[->] (y)--(z);
  \node[lb,left=1pt of x] {$\mathcal{M}$:};
  \node[nz] (ux) at (-1mm,-25mm) {$u_{x}$}; \draw[ne] (ux)--(x);
  \node[nz] (uy) at (13mm,-25mm) {$u_{m}$}; \draw[ne] (uy)--(y);
  \node[nz] (uz) at (27mm,-25mm) {$u_{y}$}; \draw[ne] (uz)--(z);
  \draw[ae] (x)--(X);
  \draw[ae] (y)--(Y);
  \draw[ae] (z)--(Z);
\end{tikzpicture}
\end{minipage}

\caption{Experiments 4 to 6.}
\label{fig:exp_4_6}
\end{figure}
 
\paragraph{Experiment 4: Spurious mediator.}
$\mathcal{M}$ is fork-structured: $x \to m$ and $x \to y$ independently, so $y$ does not depend on $m$. The invalid $\mathcal{E}$ incorrectly posits the chain $x \to m \to y$. Specifically, $m = 2x + 1$ and $y = 3x + 2$ in $\mathcal{M}$; the invalid chain $\mathcal{E}$ posits $y = 1.5m + 0.5$, which reproduces $y = 3x + 2$ observationally but fails under interventions on $m$.
 
\paragraph{Experiment 5: Unreachable states.}
Both $\mathcal{E}$ and $\mathcal{M}$ share the chain $X \to M \to Y$ with $M \in \{0, 1, 2, 3\}$. The dynamics of $\mathcal{M}$ only ever produce $M \in \{0, 2\}$ from any input $X \in \{0, 1\}$. The valid $\mathcal{E}$ uses the threshold $M \geq 2$; the invalid $\mathcal{E}$ uses $M \geq 1$, so they disagree only at $M = 1$, which is never reached from natural inputs. In micro-space, $\mathcal{M}$ uses a threshold of 1.5 to cleanly separate the $M \in \{0, 1\}$ subspaces from the $\{2, 3\}$ subspaces.
 
\paragraph{Experiment 6: Wrong causal direction.}
$\mathcal{M}$ implements the chain $x \to m \to y$ (so $y$ depends on $m$). The invalid $\mathcal{E}$ posits the fork $x \to m$, $x \to y$ (treating $x$ as the common cause of both, ignoring $m \to y$). Specifically, $m = 2x + 1$ and $y = 1.5m + 0.5$ in $\mathcal{M}$; the invalid fork $\mathcal{E}$ posits $y = 3x + 2$, reproducing $y = 3x + 2$ observationally but failing under interventions on $m$.

%% file: appendix/theoretical.tex
\section{Compositionality of \textnormal{\normalsize CAE}}
\label{app:theoretical}

\subsection{Compositionality}
\label{ssec:compositionality}
Scientific knowledge is often organized in hierarchies: a coarse macro-model $\mathcal{M}^C$ abstracts a meso-model $\mathcal{M}^B$, which in turn abstracts a micro-model $\mathcal{M}^A$. The following proposition shows that, under the compatibility assumption stated below (Assumption~\ref{ass:compatible_pi}), \cae{} degrades gracefully across abstraction chains: intermediate errors accumulate at most additively.

\begin{assumption}[Compatible intervention distributions]
\label{ass:compatible_pi}
The intervention distributions $P_I^A$, $P_I^B$, $P_I^C$ are \emph{compatible} across levels:
\begin{itemize}
    \item For \caedown{}: the marginal of $\nu_B$ induced by $\nu_C \sim P_I^C$, $\nu_B \sim P_{\nu_C}$ equals $P_I^B$; and the marginal of $\mu_A$ induced by $\nu_B \sim P_I^B$, $\mu_A \sim P_{\nu_B}^{AB}$ equals $P_I^A$; and for every $\nu_C$, the conditional distribution of $\mu_A$ induced by $\nu_B \sim P_{\nu_C}^{BC}$, $\mu_A \sim P_{\nu_B}^{AB}$ equals $P_{\nu_C}^{AC}$, i.e.\ the chain of compatible micro-interventions composes correctly.
    \item For \caeup{}: the pushforward $(\tau_{AB})_\ast P_I^A = P_I^B$.
\end{itemize}
A sufficient condition for the conditional compatibility is that each $P_I$ is uniform over intervention subsets and values, and that $\tau_{AB}$, $\tau_{BC}$ are surjections with equal-measure fibers (e.g., uniform surjections in discrete settings, or measure-preserving maps in continuous ones).
\end{assumption}

\begin{proposition}[Compositionality]
\label{prop:compositionality}
Let $\mathcal{M}^A$, $\mathcal{M}^B$, $\mathcal{M}^C$ be SCMs with abstraction maps $\tau_{AB}$ and $\tau_{BC}$, and let $\tau_{AC} = \tau_{BC} \circ \tau_{AB}$. Under Assumption~\ref{ass:compatible_pi} and $D = D_{\mathrm{TV}}$, both variants satisfy
\begin{equation}
\label{eq:compositionality}
\cae\!\left(\mathcal{M}^A,\mathcal{M}^C,\tau_{AC}\right)
\;\leq\;
\cae\!\left(\mathcal{M}^A,\mathcal{M}^B,\tau_{AB}\right)
+
\cae\!\left(\mathcal{M}^B,\mathcal{M}^C,\tau_{BC}\right).
\end{equation}
When $\cae(\mathcal{M}^A,\mathcal{M}^B,\tau_{AB})=0$ the bound is tight and equality holds.
\end{proposition}

\begin{proof}
We prove each variant in turn; both rely on the triangle inequality for $D_{\mathrm{TV}}$ and the data processing inequality (DPI) for deterministic maps.

\paragraph{\caedown{}.}
Fix a C-level intervention $\nu_C$, draw a compatible B-level intervention $\nu_B \sim P_{\nu_C}$, and draw a compatible A-level intervention $\mu_A \sim P_{\nu_B}$. Since $\tau_{AC} = \tau_{BC} \circ \tau_{AB}$, transitivity of compatibility gives $\mu_A \in \tau_{AC}^{-1}(\nu_C)$. By the conditional compatibility condition in Assumption~\ref{ass:compatible_pi}, sampling via the chain $\nu_B \sim P_{\nu_C}^{BC}$, $\mu_A \sim P_{\nu_B}^{AB}$ yields the correct conditional distribution $P_{\nu_C}^{AC}$, so the triple expectation below is precisely $\caedown^{AC}$.
Define three distributions over $Y_C$:
\begin{align*}
P &:= \mathcal{M}^C\!\left(Y_C \mid \mathrm{do}_{\nu_C}\right), \\
Q &:= \tau_Y^{BC}\!\left(\mathcal{M}^B\!\left(Y_B \mid \mathrm{do}_{\nu_B}\right)\right), \\
R &:= \tau_Y^{AC}\!\left(\mathcal{M}^A\!\left(Y_A \mid \mathrm{do}_{\mu_A}\right)\right)
   = \tau_Y^{BC}\!\left(\tau_Y^{AB}\!\left(\mathcal{M}^A\!\left(Y_A \mid
     \mathrm{do}_{\mu_A}\right)\right)\right).
\end{align*}
By the triangle inequality, $D_{\mathrm{TV}}(P, R) \leq D_{\mathrm{TV}}(P, Q) + D_{\mathrm{TV}}(Q, R)$. Since $\tau_Y^{BC}$ is deterministic, the DPI gives
\begin{equation*}
D_{\mathrm{TV}}(Q, R)
  \;\leq\;
  D_{\mathrm{TV}}\!\left(
    \mathcal{M}^B(Y_B\mid\mathrm{do}_{\nu_B}),\;
    \tau_Y^{AB}\!\left(\mathcal{M}^A(Y_A\mid\mathrm{do}_{\mu_A})\right)
  \right).
\end{equation*}
Combining and taking expectations over $\nu_C \sim P_I^C$, $\nu_B \sim P_{\nu_C}$, $\mu_A \sim P_{\nu_B}$:
\begin{align*}
\caedown^{AC}
  &\leq
  \underbrace{
    \mathbb{E}_{\nu_C}\mathbb{E}_{\nu_B\mid\nu_C}
    \bigl[D_{\mathrm{TV}}(P,\,Q)\bigr]
  }_{(I)}
  +
  \underbrace{
    \mathbb{E}_{\nu_C}\mathbb{E}_{\nu_B\mid\nu_C}
    \mathbb{E}_{\mu_A\mid\nu_B}
    \left[D_{\mathrm{TV}}\!\left(
      \mathcal{M}^B(Y_B\mid\mathrm{do}_{\nu_B}),\;
      \tau_Y^{AB}\!\left(\mathcal{M}^A(Y_A\mid\mathrm{do}_{\mu_A})\right)
    \right)\right]
  }_{(II)}.
\end{align*}
Term $(I)$ equals $\caedown^{BC}$ directly, since $\mu_A$ does not appear. For term $(II)$, Assumption~\ref{ass:compatible_pi} guarantees that the marginal of $\nu_B$ under $(\nu_C \sim P_I^C,\, \nu_B \sim P_{\nu_C})$ is $P_I^B$, and the marginal of $\mu_A$ under $(\nu_B \sim P_I^B,\, \mu_A \sim P_{\nu_B})$ is $P_I^A$, so term $(II)$ equals $\caedown^{AB}$, yielding~\eqref{eq:compositionality}.

\paragraph{\caeup{}.}
Draw $\mu_A \sim P_I^A$ and define $\nu_B := \tau_{AB}(\mu_A)$, $\nu_C := \tau_{BC}(\nu_B) = \tau_{AC}(\mu_A)$ deterministically. Define:
\begin{align*}
P &:= \mathcal{M}^C\!\left(Y_C \mid \mathrm{do}_{\nu_C}\right), \\
Q &:= \tau_Y^{BC}\!\left(\mathcal{M}^B\!\left(Y_B \mid \mathrm{do}_{\nu_B}\right)\right), \\
R &:= \tau_Y^{AC}\!\left(\mathcal{M}^A\!\left(Y_A \mid \mathrm{do}_{\mu_A}\right)\right).
\end{align*}
By the triangle inequality, $D_{\mathrm{TV}}(P, R) \leq D_{\mathrm{TV}}(P, Q) 
+ D_{\mathrm{TV}}(Q, R)$. The DPI applied to $\tau_Y^{BC}$ gives
\begin{equation*}
D_{\mathrm{TV}}(Q, R)
  \;\leq\;
  D_{\mathrm{TV}}\!\left(
    \mathcal{M}^B(Y_B\mid\mathrm{do}_{\nu_B}),\;
    \tau_Y^{AB}\!\left(\mathcal{M}^A(Y_A\mid\mathrm{do}_{\mu_A})\right)
  \right).
\end{equation*}
Taking expectations over $\mu_A \sim P_I^A$:
\begin{align*}
\caeup^{AC}
  &\leq
  \underbrace{
    \mathbb{E}_{\mu_A}\bigl[D_{\mathrm{TV}}(P,\,Q)\bigr]
  }_{(I)}
  +
  \underbrace{
    \mathbb{E}_{\mu_A}\!\left[D_{\mathrm{TV}}\!\left(
      \mathcal{M}^B(Y_B\mid\mathrm{do}_{\nu_B}),\;
      \tau_Y^{AB}\!\left(\mathcal{M}^A(Y_A\mid\mathrm{do}_{\mu_A})\right)
    \right)\right]
  }_{(II)}.
\end{align*}
Term $(II)$ equals $\caeup^{AB}$ directly. For term $(I)$, Assumption~\ref{ass:compatible_pi} gives $(\tau_{AB})_\ast P_I^A = P_I^B$, so the marginal of $\nu_B = \tau_{AB}(\mu_A)$ is $P_I^B$, and term $(I)$ equals $\caeup^{BC}$, yielding~\eqref{eq:compositionality}.

\paragraph{Tightness.}
We consider both variants. When $\cae^{AB}_{TV} = 0$, the AB consistency error vanishes almost surely in each case, so $\tau_Y^{AB}(\mathcal{M}^A(Y_A\mid\mathrm{do}_{\mu_A})) = \mathcal{M}^B(Y_B\mid\mathrm{do}_{\nu_B})$ a.s., and therefore $Q = R$ a.s. It follows that $D_{\mathrm{TV}}(P, R) = D_{\mathrm{TV}}(P, Q)$ a.s., and taking expectations gives $\cae^{AC}_{TV} = \cae^{BC}_{TV}$, so equality holds in~\eqref{eq:compositionality} for both variants.
\end{proof}

\begin{remark}[Extension to KL divergence]
The triangle inequality does not hold for KL divergence in general. For any three distributions $P$, $Q$, $R$:
\begin{equation*}
D_{\mathrm{KL}}(P \| R) = D_{\mathrm{KL}}(P \| Q) + D_{\mathrm{KL}}(Q \| R) + \Delta(P,Q,R),
\end{equation*}
where $|\Delta(P,Q,R)| \leq 2\,D_{\mathrm{TV}}(P,Q)\cdot\|\log(Q/R)\|_\infty$. If the densities of $Q$ and $R$ satisfy $\|\log(dQ/dR)\|_\infty \leq \alpha$ for some finite $\alpha > 0$ (i.e.\ the log-density ratio is uniformly bounded), then $|\Delta(P,Q,R)| \leq 2\alpha\, D_{\mathrm{TV}}(P,Q)$, and Pinsker's inequality ($D_{\mathrm{TV}} \leq \sqrt{D_{\mathrm{KL}}/2}$) gives
$|\Delta| = O\!\left(\sqrt{\cae^{BC}_{KL}}\right)$, so \begin{equation*}
\cae_{KL}\!\left(\mathcal{M}^A,\mathcal{M}^C\right)
\;\leq\;
\cae_{KL}\!\left(\mathcal{M}^A,\mathcal{M}^B\right)
+\cae_{KL}\!\left(\mathcal{M}^B,\mathcal{M}^C\right)
+O\!\left(\sqrt{\cae_{KL}\!\left(\mathcal{M}^B,\mathcal{M}^C\right)}\right),
\end{equation*}
under the same compatibility assumption. The qualitative conclusion is unchanged: well-controlled intermediate abstractions yield a well-controlled composed abstraction.
\end{remark}

%% file: appendix/systems.tex
\section{Benchmark Systems}
\label{app:systems}

For each system, we describe the high-level model~$\mathcal{E}$, the low-level model~$\mathcal{M}$, the abstraction map~$\tau$, the invalid conditions used to test discrimination, and key implementation details. All metrics are evaluated on a fixed number $N$ of samples and averaged over $n_\text{runs}$. Their values, given below, were chosen to keep computational costs similar across systems.
 
\subsection{Logic Circuit}
A 2-bit ripple-carry adder evaluated at the wire level ($\mathcal{M}$) and the gate level ($\mathcal{E}$). This is the simplest discrete-to-discrete abstraction in the benchmark.

\begin{itemize}[leftmargin=*]
    \item $\mathcal{E}$: An SCM with integer-valued variables: two 2-bit operands ($\mathtt{A}$, $\mathtt{B} \in \{0,\ldots,3\}$), a 1-bit carry-in, a 1-bit internal carry, a 2-bit sum output, and a 1-bit carry-out. Each variable is a deterministic Boolean arithmetic function of its parents.

    \item $\mathcal{M}$: A Boolean netlist simulator over individual wire states. Each wire carries a float value in $\{0.0, 1.0\}$, and gates are simulated by iterating until stable. Interventions force specific wire values before propagation.

    \item \textbf{Abstraction map:} The coarse-graining map groups pairs of 1-bit wires or single wires into variables of $\mathcal{E}$. The value map maps each integer label to the Cartesian product of unit-width subspaces: $(-0.1, 0.1)$ for bit-0 and $(0.9, 1.1)$ for bit-1. Internal gate wires are declared as internal variables and excluded from the $\Phi$ set. Because every wire is either explicitly coarse-grained or declared internal, the $\Phi$ set is empty.

    \item \textbf{Invalid conditions:}\\
    (i)~\emph{Fail}: $\mathcal{E}$ replaces XOR with OR in both the internal-carry and the sum logic.\\
    (ii)~\emph{Inverted internal}: $\mathcal{E}$ inverts $\mathtt{Internal\_Carries}$ but compensates downstream, so the final output is correct, but the intermediate is wrong.\\
    (iii)~\emph{Noise}: Gaussian noise ($\sigma=0.4$) is added to wire voltages.

    \item \textbf{Implementation details:} Metrics are evaluated over $N = 1000$ samples and averaged over 100 runs.
\end{itemize}
 
\subsection{Transistor Circuit}
 
A CMOS half-adder (sum + carry) simulated at the SPICE transistor level ($\mathcal{M}$) and the Boolean gate level ($\mathcal{E}$). This tests a continuous-to-discrete abstraction in which thresholding continuous voltages yields Boolean gate values.
 
\begin{itemize}[leftmargin=*]
    \item $\mathcal{E}$: A Boolean gate network with variables $\{a, b, \mathtt{sum}, \mathtt{carry}\} \in \{0, 1\}$, implementing XOR for sum and AND for carry.

    \item $\mathcal{M}$: A PySpice SPICE simulation of the CMOS half-adder topology. The topology uses four NAND gates and one inverter to implement the half-adder, and produces continuous node voltages for all circuit nodes (inputs, outputs, internal NAND wires, transistor junctions, and power rail). Internal nodes are declared as internal variables and excluded from the $\Phi$ set.
 
    \item \textbf{Abstraction map:} The coarse-graining map associates $\{a, b, \mathtt{sum}, \mathtt{carry}\}$ nodes to the corresponding variables of $\mathcal{E}$. The value map maps label 0 to voltages in $(-0.5, 0.5)$ V and label 1 to $(4.5, 5.5)$ V. Voltages outside these ranges are unmapped.
 
    \item \textbf{Invalid conditions:}\\
    (i)~\emph{Fail}: $\mathcal{E}$ uses OR instead of XOR for sum.\\
    (ii)~\emph{Noise}: Gaussian noise ($\sigma=1.5$~V) is added to all circuit nodes (inputs, outputs, internal wires, and the power rail).

    \item \textbf{Implementation details:} Metrics are evaluated over $N = 250$ samples and averaged over 50 runs.
\end{itemize}
 
\subsection{Gas Simulation}
 
An $N$-particle Lennard-Jones fluid ($\mathcal{M}$) evaluated against the ideal gas law or Van der Waals equation of state ($\mathcal{E}$). This is an extreme example of global aggregation: every variable of $\mathcal{E}$ aggregates over all $N$ particles. A detailed case study of this system is provided in Appendix~\ref{app:gas_study}.

\begin{itemize}[leftmargin=*]
    \item $\mathcal{E}$: The ideal gas law $PV = NT$ (reduced units) or the Van der Waals equation $(P + a\rho^2)(1/\rho - b) = T$, with causal variables $\{P, V, T\}$. Unless specified, the causal graph is reduced to the acyclic $(V, T) \to P$ forward macroscopic mapping (NVT ensemble).

    \item $\mathcal{M}$: An $N$-body molecular dynamics simulation using the 12-6 Lennard-Jones potential ($N=128$ particles), velocity-Verlet integration, Langevin thermostat, and Berendsen barostat. Long-range tail corrections are applied analytically to the pressure virial to account for truncation of the potential at cutoff $r_c=3.0\sigma$ in reduced units. $\mathcal{M}$ outputs the measured $P$, $T$, $V$ alongside particle positions and velocities. An energy minimization step (steepest descent, up to 500 iterations) is applied after lattice initialization to remove particle overlaps before dynamics begin.

    \item \textbf{Abstraction map:} The coarse-graining map is the identity on $\{P, V, T\}$ (these are directly computed macroscopic observables returned by the simulator: T via kinetic energy, P via the virial theorem, V = $\text{BoxLength}^3$). The value map is the identity. The underlying particle substrate (positions, velocities, box length) is the state $\mathcal{M}$ integrates; it is not part of the macro description and is excluded from $\tau_X$, so $\Phi=\emptyset$ and the gas system carries no faithfulness test.

    \item \textbf{Invalid conditions:}\\
    (i)~\emph{Wrong $\alpha$}: the ideal gas $\mathcal{E}$ uses an incorrect temperature exponent.\\
    (ii)~\emph{High density}: particle density is raised outside the ideal gas regime.\\
    (iii)~\emph{Low temperature}: temperature is reduced to a regime where the ideal gas law breaks down.

    \item \textbf{Implementation details:} Van der Waals parameters $a, b$ are calibrated against LJ fluid isotherms at the Boyle temperature ($T^* \approx 3.42$)~\citep{GLASSER2002381}. Metrics are evaluated over $N = 10$ samples and averaged over 20 runs; a custom sampler initializes lattice configurations and hands equilibration to $\mathcal{M}$.
\end{itemize}

\subsection{Predator-Prey Dynamics}

An agent-based model of predator-prey dynamics ($\mathcal{M}$), with optional spatial, stochastic, and aging variants, evaluated against the Lotka-Volterra ODE system ($\mathcal{E}$). The valid baseline used for calibration is the non-spatial configuration. The abstraction is a continuous identity mapping over total population counts: $\mathcal{M}$ outputs aggregate agent counts, which are the macro-variables.

\begin{itemize}[leftmargin=*]
    \item $\mathcal{E}$: The Lotka-Volterra equations, with parameters $\{\alpha, \beta, \delta, \gamma\}$ calibrated against the ideal (non-spatial, deterministic) ABM. The LV equations are integrated via the forward Euler method over 50 steps, matching the ABM's timestep count.

    \item $\mathcal{M}$: An agent-based model (ABM) of predator-prey dynamics in which individual prey and predators are tracked. Each timestep, births, predation events, and deaths are computed from per-capita probabilities. The simplest configuration (non-spatial, deterministic rates, no aging) approximates the LV equations and is used for calibration.

    \item \textbf{Abstraction map:} The coarse-graining map is the identity on $\{\mathtt{prey\_t}, \mathtt{predator\_t}, \mathtt{final\_populations}\}$. The value map is the identity.

    \item \textbf{Invalid conditions:}
    (i)~\emph{Wrong $\alpha$}: prey reproduction rate in $\mathcal{E}$ is 1.5 times higher than in $\mathcal{M}$.\\
    (ii)~\emph{Spatial}: $\mathcal{M}$ uses spatial dynamics on a $20\times20$ grid, where agents move to and interact with neighboring cells.\\
    (iii)~\emph{Stochastic}: $\mathcal{M}$ uses stochastic (binomial demographic) reproduction. The noise is mean-unbiased (its expectation equals the deterministic per-capita rate) so this is a zero-mean-shift, variance-only contrastive that mean-based baselines under-detect.\\
    (iv)~\emph{Aging}: agents have finite lifespans.\\
    (v)~\emph{Noise}: Gaussian noise ($\sigma=5$) on output populations.\\
    (vi)~\emph{Complex}: all three $\mathcal{M}$ extensions simultaneously active.

    \item \textbf{Implementation details:} Metrics are evaluated over $N = 50$ samples and averaged over 100 runs. LV parameters are calibrated via grid search.

\end{itemize}

\subsection{Heat Equation (1D)}
 
Brownian particle diffusion ($\mathcal{M}$) is abstracted to a finite-difference heat equation ($\mathcal{E}$). This tests a spatial aggregation abstraction.

\begin{itemize}[leftmargin=*]
    \item $\mathcal{E}$: A finite-difference heat equation solver on a uniform grid of $K$ bins, parameterized by the physical diffusion coefficient $\alpha$ (distinct from the solver's internal dimensionless update coefficient $\alpha\,\Delta t/\Delta x^2$). The finite-difference solver uses Neumann boundary conditions (zero flux at both walls), consistent with the reflective particle boundaries in $\mathcal{M}$.

        \item $\mathcal{M}$: A system of $N=1000$ Brownian particles diffusing in a 1D box $[0, L]$ with reflective boundaries over 200 time steps.

    \item \textbf{Abstraction map:} The value map bins particle positions into a normalized density histogram (one bin per cell in $\mathcal{E}$), and grounding samples particle positions from that histogram as a probability density function.

    \item \textbf{Invalid conditions:}\\
    (i)~\emph{Fail}: the diffusion coefficient in $\mathcal{E}$ is set to a wrong value of zero\\
    (ii)~\emph{Noise}: Gaussian noise ($\sigma=0.15$) on the outputs of $\mathcal{M}$.

    \item \textbf{Implementation details:} Metrics are evaluated over $N=30$ samples and averaged over 50 runs. $K=50$ bins. The diffusion coefficient $\alpha=0.1$ is set to a known value matching the particle simulation's diffusion constant. Initial conditions are sampled as random Gaussian profiles (mean $\sim \mathrm{U}(0.3, 0.7)$, width $\sim \mathrm{U}(0.05, 0.15)$), grounded to particle positions and re-abstracted to obtain $\mathcal{E}$-compatible inputs.
\end{itemize}

\subsection{Heat Equation (2D)}
 
A phonon lattice model ($\mathcal{M}$) is abstracted to a 2D heat PDE ($\mathcal{E}$). This tests a spatial aggregation abstraction.

\begin{itemize}[leftmargin=*]
    \item $\mathcal{E}$: A 2D heat PDE solver with Neumann boundaries and a calibrated diffusion coefficient $\alpha$.

    \item $\mathcal{M}$: A $16\times16$ phonon lattice model with harmonic springs, phonon scattering (rate $p_\text{scatter}=0.2$), and $N_\text{avg}=10$ simulation averages per sample.

    \item \textbf{Abstraction map:} The value map abstracts the kinetic energy map via Gaussian smoothing ($\sigma = 0.05 \times \text{GridSize}$) and normalization; grounding uses a parameterized Gaussian source. The source parameters $(\mathtt{source\_x}, \mathtt{source\_y}, \mathtt{source\_E})$ use identity grounding.

    \item \textbf{Invalid conditions:}\\
    (i)~\emph{Fail}: the diffusion coefficient in $\mathcal{E}$ is set to a wrong value of $10\times$ the calibrated value\\
    (ii)~\emph{Noise}: Gaussian noise ($\sigma=50$) on the output of $\mathcal{M}$.

    \item \textbf{Implementation details:} Metrics are evaluated over $N=5$ samples and averaged over 10 runs. The diffusion coefficient $\alpha$ is calibrated via scalar optimization. The explicit Euler PDE solver uses adaptive sub-stepping: when the stability criterion $\alpha \Delta t/\Delta x^2 > 0.2$ is violated, the time step is subdivided accordingly.
\end{itemize}

\subsection{Ising Model}

A hybrid molecular-dynamics/Monte-Carlo simulator in which atomic positions can evolve via Newtonian dynamics when the vibrational temperature is $>0$, and spins are updated via Metropolis sweeps with distance-dependent coupling $J_\text{eff} = J_0 \exp(-\lambda |r - r_0|)$. The motivating question is whether the rigid-lattice abstraction remains valid once atoms vibrate (which would shift $J_\text{eff}$); in this benchmark, we evaluate $\mathcal{M}$ at vibrational temperature $0$, so atoms stay on their lattice sites and $J_\text{eff}=J_0$, isolating the abstraction's validity for the static lattice.

\begin{itemize}[leftmargin=*]
    \item $\mathcal{E}$: A rigid $L \times L$ Ising model simulated via Metropolis-Hastings MCMC with coupling constant $J$. The causal variables are $\{\mathtt{Temperature}, \mathtt{ExternalField}, \mathtt{PredictedMagnetization}\}$, all continuous.

    \item $\mathcal{M}$: A hybrid molecular-dynamics/Monte-Carlo simulator in which atomic positions evolve via Newtonian dynamics when vibrational temperature $>0$, and spins are updated via Metropolis sweeps with distance-dependent coupling $J_\text{eff} = J_0 \exp(-\lambda |r - r_0|)$.

    \item \textbf{Abstraction map:} The coarse-graining map is the identity on all three continuous variables. The value map is the identity.

    \item \textbf{Invalid conditions:}\\
    (i)~\emph{Fail}: the coupling constant in $\mathcal{E}$ is set to $J=2$ (double the true value).\\ (ii)~\emph{Noise}: Gaussian noise ($\sigma=0.05$) on the magnetization output of $\mathcal{M}$.
     
    \item \textbf{Implementation details:} $L=8$ and the vibrational temperature is fixed to $0$ for all conditions (rigid lattice), so the invalid conditions are the coupling change ($J=2$) and the output noise above rather than a vibrational-breakdown regime. Metrics are evaluated over $N=30$ samples and averaged over 150 runs. Sampling is performed over $[\mathtt{Temperature}, \mathtt{ExternalField}]$.
\end{itemize}
 
\subsection{Tracr Transformer}
 
A transformer mechanistically compiled from a RASP sort-rank program ($\mathcal{M}$) evaluated against the symbolic program itself ($\mathcal{E}$). In this system, $\mathcal{M}$ is a neural network.

\begin{itemize}[leftmargin=*]
     \item $\mathcal{E}$: A symbolic sort-rank program over sequences of length $\ell=3$: $\mathtt{rank}_i = |\{j \neq i : \mathtt{token}_j < \mathtt{token}_i\}|$. Variables: $\mathtt{token}_i \in \{1,\ldots,5\}$ and $\mathtt{rank}_i \in \{0,\ldots,4\}$ (the value map and rank prior use five labels; only $\{0,1,2\}$ are reachable for $\ell=3$).

     \item $\mathcal{M}$: The Tracr-compiled JAX transformer, exposing two multi-dimensional micro-variables: $\mathtt{tokens} \in \mathbb{R}^\ell$ (float encoding of token indices) and $\mathtt{ranks} \in \mathbb{R}^\ell$ (float encoding of rank values).

     \item \textbf{Abstraction map:} The coarse-graining map slices $\mathtt{token}_i \to \mathtt{tokens}[i]$ and $\mathtt{rank}_i \to \mathtt{ranks}[i]$. The value map uses $[(v-1)-0.5, (v-1)+0.5]$ for token label $v$ (tokens are 0-indexed internally: label $v=1\to 0.0, v=5\to 4.0$) and $[r-0.5, r+0.5]$ for rank label $r$. Tracr guarantees lossless round-trip encoding: grounding and abstraction are exact inverses for integer labels.
 
    \item \textbf{Invalid conditions:}\\
    (i)~\emph{Fail}: $\mathcal{E}$ uses $>$ instead of $<$ in the rank comparison (reversed ranks).\\ (ii)~\emph{Noise}: Gaussian noise ($\sigma=0.3$) on the residual stream before read-out.
     
    \item \textbf{Implementation details:} Metrics are evaluated over $N=100$ samples and 150 runs. The sequence length is $\ell=3$, the vocabulary $\{1,\ldots,5\}$. Inputs are formed as $[\mathtt{BOS}] + [\mathtt{vocab\_labels}]$; the BOS output position is discarded from the decoded result.
 \end{itemize}

\subsection{Gene Regulatory Network (GRN)}
 
The segment polarity gene regulatory network of \citet{Sanchez2008} ($\mathcal{M}$) evaluated against a simple two-variable Boolean causal rule ($\mathcal{E}$). $\mathcal{M}$ is a multi-valued network (not a continuous ODE model) to which we apply one synchronous update step.

\begin{itemize}[leftmargin=*]
    \item $\mathcal{E}$: A two-variable Boolean SCM: $\mathtt{wg\_src} \in \{0, 1\}$ (Wg signal active?) $\to$ $\mathtt{fz\_tgt} \in \{0, 1\}$ (target Fz on?). Rule: $\mathtt{fz\_tgt} = \mathtt{wg\_src}$.
 
    \item \textbf{Low-level model:} One synchronous update step of the 6-cell segment polarity network. Variables are multi-valued (e.g.\ Wg $\in \{0, 1, 2\}$), and updates follow GINsim logical parameter semantics. Micro-variables exposed: $\mathtt{wg\_c2}$ (Wg level in cell 2, input) and $\mathtt{fz\_c1}, \mathtt{fz\_c3}, \mathtt{fz\_c4}, \mathtt{fz\_c5}$ (Fz in various cells, outputs).

    \item \textbf{Abstraction map:} The coarse-graining map links $[\mathtt{wg\_c2}] \to \mathtt{wg\_src}$ and $[\mathtt{fz\_c1}, \mathtt{fz\_c3}] \to \mathtt{fz\_tgt}$ (true neighbors of cell 2). The value map ($\mathtt{WgFzValueMap}$) thresholds: $\mathtt{wg\_src}=1$ iff $\mathtt{wg\_c2} \geq 1.9$ (effectively $=2$ for integer-valued states); $\mathtt{fz\_tgt}=1$ iff any target Fz $> 0.5$.

    \item \textbf{Invalid conditions:}\\
    (i)~\emph{Wrong map}: the coarse-graining links $[\mathtt{fz\_c4}, \mathtt{fz\_c5}]$ (non-neighbors) instead of $[\mathtt{fz\_c1}, \mathtt{fz\_c3}]$.\\
    (ii)~\emph{Wrong $\mathcal{E}$}: the causal rule is reversed ($\mathtt{fz\_tgt} = 1 - \mathtt{wg\_src}$).\\
    (iii)~\emph{Noise}: Gaussian noise ($\sigma=0.4$) on Fz outputs.

    \item \textbf{Implementation details:} Metrics are evaluated over $N=200$ samples and averaged over 200 runs. Sampling is performed over $[\mathtt{wg\_src}]$.
\end{itemize}

\subsection{MOS 6502 CPU}
 
The MOS 6502 CPU, modeled at three discrete abstraction levels: the Visual6502 digital netlist, the manually extracted gate-level implementation, and a corresponding ISA-level simulation. All three levels are discrete and differ in their internal simulator fidelity: the transistor level uses binary switch segments, the gate level uses Boolean gate outputs, and the ISA level uses integer register and flag values (0--255). The causal variables exposed to the abstraction are, at every level, the shared integer register/flag interface (so $a$ and $\tau_X$ are identity on these values, see the abstraction-map note below); individual transistor segments or gate signals are never themselves abstraction variables. We evaluate abstractions between those three layers independently.

\begin{itemize}[leftmargin=*]
    \item \textbf{Highest-level model (ISA):} The ISA semantics of selected 6502 instructions, expressed as a causal SCM over register variables (accumulator, index registers, status flags) implemented by the ISA simulator (\href{https://github.com/ivop/fake6502}{\texttt{ivop/fake6502}}, commit 09fc542).

    \item \textbf{Intermediate-level model (gate):} A gate-level simulator (\texttt{\href{https://github.com/ivop/break6502}{ivop/break6502}}, commit 922af64), which implements the decoder PLA and datapath at the Boolean gate level.

    \item \textbf{Lowest-level model (transistor):} A transistor-level digital logic simulator (\href{https://github.com/mist64/perfect6502}{\texttt{mist64/perfect6502}}, commit 6c1f1a5), extracted directly from the 6502 CPU via the \href{http://visual6502.org/}{Visual6502} project. This implements a segment model in which each transistor is a binary switch.
 
    \item \textbf{Abstraction maps:} All simulators share the same interface: integer register values 0--255 per variable. The coarse-graining map and value maps are both effectively identity on these integer values. All simulators mask the processor status register P with $\mathtt{0xCF}$ (clearing bits 4 and 5) before comparison, since bit 4 (the B-flag) only reflects how P was pushed and is not a persistent flag; bit 5 is unused/hardwired 1. Without this mask, all status-affecting instructions would appear to disagree.
    
    \item \textbf{Invalid conditions:}\\
    (i)~\emph{Broken gate $\leftrightarrow$ ISA}: accumulator bit 7 is stuck at 0 in $\mathcal{M}$.\\
    (ii)~\emph{Broken transistor $\leftrightarrow$ gate}: same fault applied at the transistor level.\\
    (iii)~\emph{Broken transistor $\leftrightarrow$ ISA}: the same accumulator bit 7 stuck-at-0 fault at the transistor level, evaluated against the ISA high-level model.
 
    \item \textbf{Implementation details:} The abstraction is evaluated on a reduced set of single-instruction tests covering implied and immediate addressing modes. Metrics are evaluated over $N=200$ samples and averaged over 100 runs (per condition).

\end{itemize}

%% file: appendix/baseline_details.tex
\section{Baseline Metric Implementation Details}
\label{app:metric_impl}
 
We document design choices, normalization functions, and known limitations for the baseline metrics. All metrics receive the same sampler, $n$ samples, and output variable list as the \cae{} variants.

\paragraph{Observational metrics (MSE, RMSE, NMSE, $R^2$, $L^2$, JSD, KL, MMD, HSIC, VarDecomp).}
These are evaluated via a task wrapper that restricts interventions to root variables only and compares outputs at the designated output nodes. For the asymmetric or normalized metrics (KL, NMSE, $R^2$, $L^2$), the high-level model $\mathcal{E}$ is the normalization reference: KL computes $\mathrm{KL}(\mathcal{E}\,\|\,\mathcal{M})$, and NMSE, $R^2$, and $L^2$ divide by the variance (resp.\ squared norm) of $\mathcal{E}$'s outputs.
 
\textit{JSD/KL:} Uses adaptive histogram binning into $\lceil\log_2 n + 1\rceil$ bins, capped at 20. KL adds Laplace smoothing ($\varepsilon=10^{-10}$). When samples collapse to a point mass, JSD returns 0. The adaptive bin count depends on $n$, so the metric definition (and not only its estimate) varies with sample size; their power curves therefore conflate estimator variance with this definitional drift.
 
\textit{MMD:} Uses an RBF kernel and bandwidth $\sigma = \sqrt{\smash[b]{\text{median}(\|x_i-x_j\|^2)/2}}$ from the pooled sample (median heuristic). The unbiased $\widehat{\text{MMD}}^2$ estimator is clipped to $[0,\infty)$, then passed through $\tanh$.
 
\textit{HSIC:} Measures whether the residuals $R = \mathcal{M}(u) - \mathcal{E}(u)$ are statistically independent of the outputs of $\mathcal{E}$ (used as a proxy for the inputs $u$). The centered kernel alignment estimator $\widehat{\text{HSIC}} = \mathrm{tr}(HK_XH \cdot HK_RH)/(n-1)^2$ is computed, where $K_X$ is the RBF kernel matrix of the outputs of $\mathcal{E}$ and $K_R$ is the RBF kernel matrix of residuals; both kernel matrices are centered by $H = I - \mathbf{1}\mathbf{1}^\top/n$. The estimator is normalized by the product of the Frobenius norms of $HK_XH$ and $HK_RH$. Requires $\geq 20$ samples; returns NaN below this threshold.
 
\textit{$R^2$:} Normalized as $(1-R^2)/2$, clipped to $[0,1]$.
 
\textit{VarDecomp:} Reports $\mathrm{Var}(y-\hat{y})/\mathrm{Var}(y)$, which is robust to constant offsets. This is a fraction-of-variance-unexplained statistic, not a Sobol/ANOVA partition.

\textit{$L^2$:} Reports the relative squared error $\mathrm{mean}(\|y-\hat{y}\|^2/\|y\|^2)$ --- an NMSE-like ratio, not a Euclidean distance --- passed through $\tanh$.
 
\paragraph{Temporal metrics (TrajMSE, DTW, Autocorr, Spectral, SINDy).}
These are implemented for the predator-prey system in a way that rolls out both $\mathcal{E}$ and $\mathcal{M}$ from shared initial conditions.
 
\textit{DTW:} Standard dynamic programming with Euclidean local cost, normalized by the sum of sequence lengths $|s_1|+|s_2|$.
 
\textit{SINDy:} For each sampled trajectory, we extract consecutive macro-step pairs $(X_t, X_{t+1})$ from the trajectory of $\mathcal{M}$ (ABM), where each step corresponds to one full model epoch. The observed epoch change is $\dot X_t = X_{t+1} - X_t$. The predicted change in $\mathcal{E}$ is $\hat{\dot X}_t = \nu(X_t) - X_t$, where $\nu(X_t)$ runs the LV equations for one full epoch starting from the actual ABM state $X_t$. Only steps where both species exceed 1 individual are scored. The normalized Frobenius residual $\|\dot X - \hat{\dot X}\|_F \,/\, \|\dot X\|_F$ is mapped through $\tanh(\cdot)$.
 
\paragraph{Sobol first-order indices:} The Saltelli estimator shares the same $A$ and $B$ sample matrices for $\mathcal{E}$ and $\mathcal{M}$, making indices directly comparable. The metric reports mean $|S_i^\mathcal{E} - S_i^\mathcal{M}|$ over input variables. In several benchmark systems, we evaluate Sobol indices on fewer than $100\times k$ samples due to the shared compute budget; reported Sobol scores should therefore be read as under-sampled estimates, not as evidence that the metric is inherently uninformative.
 
\paragraph{Infidelity (two-model attribution adaptation):} For each abstract input variable $i$, a finite-difference perturbation of size $\delta_i$ is applied to the abstract label: $\delta_i = \max(\delta_\text{frac} \cdot |x_i|,\, \delta_\text{min})$ for continuous labels (with $\delta_\text{frac}=0.1$, $\delta_\text{min}=10^{-3}$) and $\delta_i = 1$ for discrete (integer) labels. Both $\mathcal{E}$ and abstracted $\mathcal{M}$ are evaluated at the original and perturbed inputs; per-input attributions $\phi_i = (\text{output}_\text{perturbed} - \text{output}_\text{original})/\delta_i$ are computed for each model and compared by squared difference, averaged over input variables; multi-dimensional outputs are first mean-collapsed to a per-variable scalar before forming $\phi_i$. The metric is normalized via $\tanh$. For the 6502 CPU, only register inputs $\{A, X, Y, S, P\}$ are perturbed; opcode and operand bytes are excluded because a perturbation of $\delta=1$ changes the instruction identity, producing spurious attribution disagreement even for correct abstractions.
 
\paragraph{Relational fidelity:} Pearson $r$ per output dimension, averaged, mapped to $(1-r)/2\in[0,1]$.
 
\paragraph{Mallows' $C_p$:} $p = 2 \times$ the number of scored output variables of $\mathcal{E}$; $\hat\sigma^2$ is estimated as the mean variance of the outputs of $\mathcal{M}$ across output dimensions, serving as a reference noise floor. Normalized as $\tanh(\max(0,C_p-p)/n)$.
 
\paragraph{Structural deviation and causal sensitivity index:} Structural deviation perturbs each parameter of $\mathcal{E}$ by 1\% relative and measures $|\caeupnf(\mathcal{E}_\text{perturb}) - \caeupnf(\mathcal{E}_\text{base})|$. Causal sensitivity zeros each parameter instead, but measures the same quantity. Neither metric divides by the baseline \caeupnf score.
 
\paragraph{IB vs.\ CIB Lagrangian:} IB uses passive co-occurrence statistics for both $I(X;T)$ and $I(T;Y)$; a spurious correlate of $Y$ in $\mathcal{M}$ can inflate $I(T;Y)$ even without causal influence. CIB replaces $I(T;Y)$ with $\caeupnf(Y\mid\mathrm{do}(T))$, estimated as $H(Y)-H_c(Y\mid\mathrm{do}(T))$ where the conditional entropy averages over 20 label bins under direct interventions. Both Lagrangians are normalized to $[0,1]$: the IB result is shifted by $\beta H(Y)$ and divided by $H(X) + \beta H(Y)$; the CIB result is shifted by $\beta$ and divided by $H(X) + \beta$, where $H(\cdot)$ denotes histogram entropy of the respective variable.
 
\paragraph{Complexity shift:} For each sample, both the original input $x$ and a Gaussian-perturbed input $x + \varepsilon$ ($\varepsilon \sim \mathcal{N}(0, 0.01^2)$) are run through both models. The complexity shift for each model is $(K(\text{output}_{x+\varepsilon}) - K(\text{output}_x)) / K(\text{input})$, where $K(\cdot)$ denotes the zlib level-9 compressed byte length. The metric reports the absolute difference between the mean complexity shifts of $\mathcal{E}$ and $\mathcal{M}$. Because the proxy depends on float formatting and output dimensionality, complexity-shift values are not comparable across systems and are interpreted only within-system as an ordinal signal.
 
\paragraph{Symbion:} Exhaustively enumerates all root-variable label combinations; applicable
only to systems with finite discrete label spaces.
 
\paragraph{Macroscopic invariance vs.\ BCC:} Both compare abstracted outputs of $\mathcal{M}$ for two micro-states that share the same abstract label. Macroscopic invariance perturbs a single variable in isolation: it intervenes on one variable at a time, drawing two micro-states from the subspace of each label while leaving the rest of the state unspecified. BCC instead tests consistency within a complete system state: each iteration draws a joint label assignment over all mapped input (root) variables, grounds the whole state to a micro-realization, then re-grounds it to a second micro-realization of the same complete label vector, and compares the abstracted outputs read from the non-root behavioral variables. For both discrete and continuous value maps, BCC skips pairs whose grounding is the identity, as these are trivially consistent.
 
\paragraph{DCC:} Always uses hard label equality for the next-state comparison.
 
\paragraph{IIA:} The interchange copies raw micro-values of $\mathcal{M}$ from the source run (not grounded preimages of the source abstract label). This tests whether the actual internal micro-state of $\mathcal{M}$ is causally exchangeable, which is a stronger condition than checking any preimage of the source label. This distinction vanishes for identity value maps.
 
\paragraph{Probing:} The probe is trained on $n_\text{train}$ samples (default 200, overridden per system) and evaluated on a held-out test set. When the probe is scored with hard label matching (Tracr), its predictions are rounded to the nearest integer label before scoring; when it is scored with MSE (logic circuit and GRN, whose integer labels are treated as regression targets, as well as the continuous identity value maps), the probe is a standard linear regression whose continuous predictions are compared to the labels without rounding.

\subsection{Computation Costs}

All experiments were run on a consumer-grade laptop equipped with 16 GB of RAM and an Intel i5-13600H CPU. The computation time required to produce experimental results is estimated as such:
\begin{itemize}
    \item Evaluation of baseline metrics on benchmark systems for valid and invalid abstraction: approx. 100 hours, of which nearly half was allocated to system 3 (gas simulation)
    \item Measurements of statistical power and convergence: approx. 50 hours, of which half for system 3
    \item Scaling laws for Tracr: Approx. 5 hours
\end{itemize}

In addition, functional tests were implemented and run to check for correctness of the CPU implementations. Their total runtime is estimated at approx. 50 hours.

Finally, preliminary experiments, including debugging and prototyping, are estimated to have required twice as much compute as the production of final results (approx. 300 hours).

%% file: appendix/case_study.tex
\section{Case Study: Gas simulation (system 3)}
\label{app:gas_study}

We report additional experiments for the gas simulation system, which tests whether thermodynamic laws emerge as valid causal abstractions of a Lennard-Jones fluid simulation. Unless otherwise stated, experiments are performed at the Boyle temperature $T_\text{Boyle}\approx3.418$ \citep{GLASSER2002381} and with a low density of $\rho=0.05$, where inter-particle interactions are rare and ideal thermodynamic behavior is expected.

We evaluate two high-level models: the ideal gas law (IGL, $PV \propto T$) and the Van der Waals equation (VdW, $(P + a\rho^2)(1/\rho - b) \propto T$). These models represent successive levels of physical refinement: the IGL assumes non-interacting point particles, while VdW is a historical correction that explicitly accounts for finite particle volume and pairwise attraction. Neither is exact, but they serve as a useful pair to test whether \cae{} can discriminate the quality of competing valid abstractions.

For each equation, we consider three causal directions corresponding to different
thermodynamic ensembles: NVT ($P \leftarrow f(V, T)$), NPT ($V \leftarrow f(P, T)$), and
a non-standard inverse problem PVT ($T \leftarrow f(P, V)$), implemented via a
proportional-gain temperature controller.

\subsection{Valid Abstractions}

Table~\ref{tab:gas_main_scores} confirms that both models achieve near-zero \cae{} across all three causal directions in the dilute limit. This demonstrates that at low density, the validity of the statistical mechanical mapping matters more than the precise choice of equation of state. The VdW equation is marginally more faithful across most conditions, reflecting its tighter fit to the LJ fluid.

\begin{table}[!htbp]
    \centering
    \begin{tabular}{l|c|c|c}
        \textbf{\cae{}}
            & NVT ($P \leftarrow f(V, T)$) & NPT ($V \leftarrow f(P, T)$) & PVT ($T \leftarrow f(P, V)$) \\
        \hline
        Ideal gas law          & 0.0231 & 0.0397 & 0.0278 \\
        Van der Waals equation & 0.0135 & 0.0520 & 0.0435
    \end{tabular}
    \caption{\caedown{} for two high-level gas models across three causal directions (10 samples per evaluation).}
    \label{tab:gas_main_scores}
\end{table}

\subsection{Destructive Perturbations}

We next systematically violated the physical assumptions required for thermodynamic laws to emerge from particle dynamics. Results are shown in Figure~\ref{fig:destructive_gas} and Table~\ref{tab:destructive_gas}.

\paragraph{Density sweep ($\rho$).}
Increasing density from $\rho = 0.05$ to $\rho = 0.85$ in NVT mode introduces progressively stronger inter-particle interactions. The IGL error rises steadily, while VdW remains below $0.05$ for significantly longer, exploiting its explicit correction terms. Both abstractions effectively fail (error $\approx 1$) for $\rho > 0.45$.

\paragraph{Temperature sweep ($T$).}
At low temperatures ($T < 1.0$), particles cluster in attractive wells, violating the homogeneity assumptions of both laws, and both fail (error $\approx 0.8$). As temperature increases, VdW recovers faster, reaching its minimum around $T \approx 2.5$. The IGL minimum falls in $[3.25, 3.75]$, consistent with the Boyle temperature where attractive and repulsive contributions cancel~\citep{GLASSER2002381}. Both models remain valid at high temperatures ($T > 4.0$, error $< 0.1$) where kinetic energy dominates.

\paragraph{Distorted exponent ($\alpha$).}
Replacing the IGL's temperature dependence with $P \propto T^\alpha$ and sweeping $\alpha \in [0.5, 1.5]$ tests sensitivity to functional form. The \cae{} correctly identifies the physical ground truth, reaching near zero at $\alpha = 1.0$ and increasing monotonically away from it: sub-linear deviations ($\alpha < 1$) degrade error nearly linearly (reaching $\approx 0.7$ at $\alpha = 0.5$), while super-linear deviations are penalized more gradually ($\approx 0.42$ at $\alpha = 1.5$).

\begin{figure}[htbp]
    \centering
    \includegraphics[width=.32\linewidth]{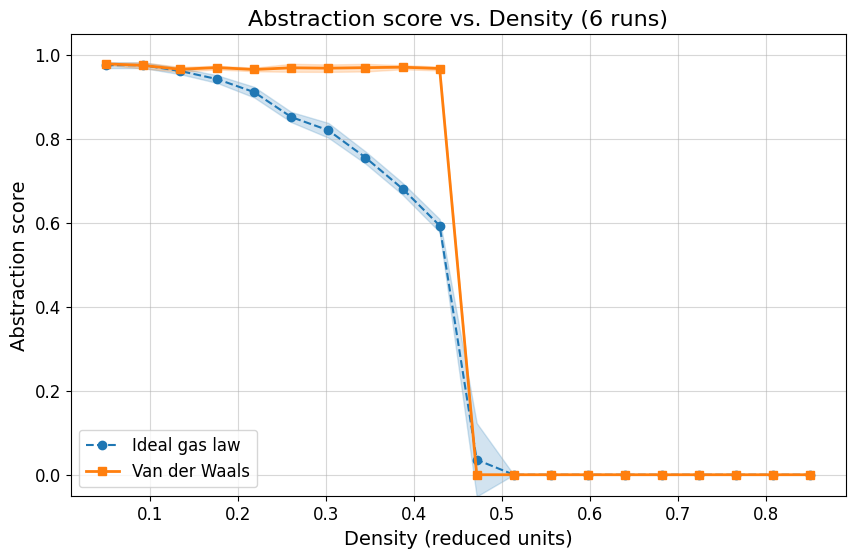}
    \includegraphics[width=.32\linewidth]{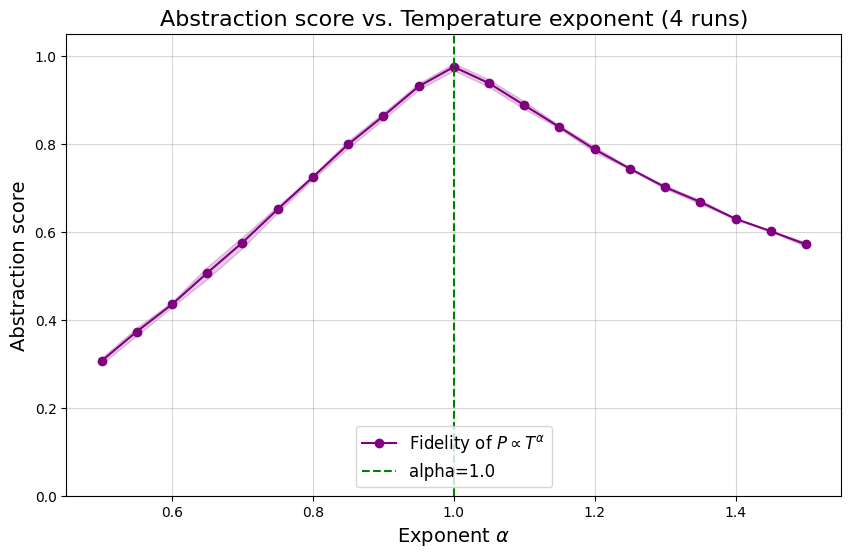}
    \includegraphics[width=.32\linewidth]{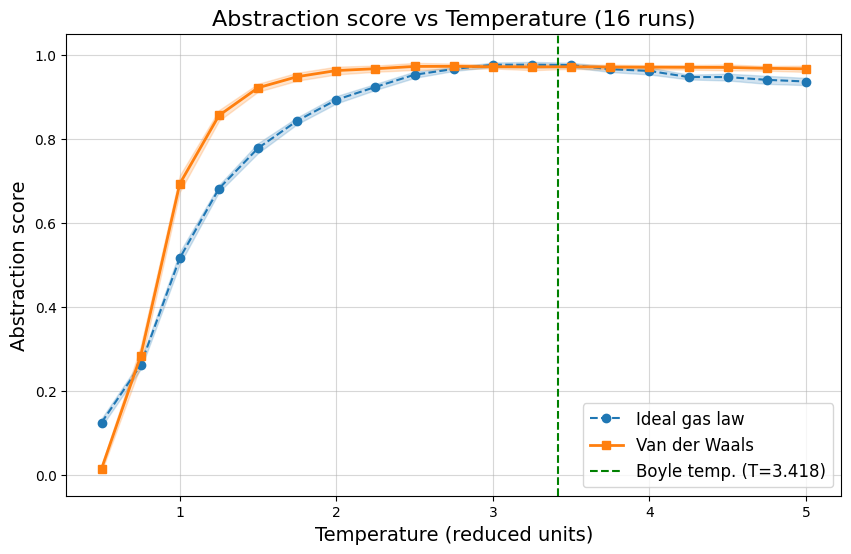}
    \caption{Average \caedown{} for IGL and VdW under (a) increasing density, (b) varying temperature exponent $\alpha$, and (c) varying initial temperature $T$. All experiments in NVT mode, 10 samples per evaluation. Shaded areas denote $\pm 1\sigma$.}
    \label{fig:destructive_gas}
\end{figure}

\begin{table}[htbp]
    \centering
    \begin{tabular}{l|c|c|c}
        \textbf{\caedown{}} & NVT ($P \leftarrow f(V, T)$) & NPT ($V \leftarrow f(P, T)$) & PVT ($T \leftarrow f(P, V)$) \\
        \hline
        Ideal gas law          & 0.5920 & 0.5432 & 0.5497 \\
        Van der Waals equation & 0.5892 & 0.5478 & 0.5551
    \end{tabular}
    \caption{\caedown{} for both models in non-equilibrium conditions (10 samples per evaluation).}
    \label{tab:destructive_gas}
\end{table}

Together, these results illustrate that \caedown{} behaves as a graded physical validity score: it correctly validates both abstractions in their regime of applicability, distinguishes the more faithful VdW model at intermediate densities and temperatures, and detects violations as soon as the underlying physical assumptions break down.